%% file: arxiv_version.tex
\documentclass{article}

\usepackage{amssymb}
\usepackage{microtype}
\usepackage{multicol}

\usepackage{graphicx}
\usepackage{caption}
\usepackage{subcaption}
\usepackage{csquotes}
\usepackage{floatrow}

\usepackage{booktabs} 
\usepackage{amsmath}
\usepackage{enumitem}
\usepackage{todonotes}
\usepackage{multirow}

\usepackage{float}
\floatstyle{plaintop}
\restylefloat{table}

\usepackage{amsthm}

\newtheorem{theorem}{Theorem}
\newtheorem{corollary}{Corollary}

\newif\ifarxivcopy
\arxivcopytrue

\newif\ifshowchanges
\showchangestrue

\usepackage{hyperref}

\usepackage[accepted]{icml2020}

\icmltitlerunning{Attentive Group Equivariant Convolutional Networks}

\begin{document}

\twocolumn[
\icmltitle{Attentive Group Equivariant Convolutional Networks}



\icmlsetsymbol{equal}{*}

\begin{icmlauthorlist}
\icmlauthor{David W. Romero}{vu}
\icmlauthor{Erik J. Bekkers}{uva}
\icmlauthor{Jakub M. Tomczak}{vu}
\icmlauthor{Mark Hoogendoorn}{vu}
\end{icmlauthorlist}

\icmlaffiliation{vu}{Vrije Universiteit Amsterdam,}
\icmlaffiliation{uva}{University of Amsterdam, The Netherlands}

\icmlcorrespondingauthor{David W. Romero}{d.w.romeroguzman@vu.nl}

\icmlkeywords{Machine Learning, ICML, Group Convolutions, Equivariance, Self-Attention, Visual Attention, Group Equivariant Attention, Group Equivariant Visual Attention}

\vskip 0.3in
]



\printAffiliationsAndNotice{}  

\begin{abstract}
Although group convolutional networks are able to learn powerful representations based on symmetry patterns, they lack explicit means to learn meaningful relationships among them (e.g., relative positions and poses). In this paper, we present \textit{attentive group equivariant convolutions}, a generalization of the group convolution, in which attention is applied during the course of convolution to accentuate meaningful symmetry combinations and suppress non-plausible, misleading ones. We indicate that prior work on visual attention can be described as special cases of our proposed framework and show empirically that our \textit{attentive group equivariant convolutional networks} consistently outperform conventional group convolutional networks on benchmark image datasets. Simultaneously, we provide interpretability to the learned concepts through the visualization of equivariant attention maps.
\end{abstract}

\section{Introduction} \label{sec:intro}
Convolutional Neural Networks (CNNs) \cite{lecun1989backpropagation} have shown impressive performance in a wide variety of domains. The developments of CNNs as well as of many other machine learning approaches have been fueled by intuitions and insights into the composition and \textit{modus operandi} of multiple biological systems \cite{wertheimer1938gestalt, biederman1987recognition, delahunt2019insect, blake2005role,v1hypothesis, delahunt2019insect}. 
Though CNNs have achieved remarkable performance increases on several benchmark problems, their training efficiency as well as generalization capabilities are still open for improvement.
One concept being exploited for this purpose is that of \emph{equivariance}, again drawing inspiration from human beings. 

Humans are able to identify familiar objects despite modifications in location, size, viewpoint, lighting conditions and background \cite{bruce1994recognizing}. In addition, we do not just recognize them but are able to describe in detail the type and amount of modification applied to them as well \cite{vonHelmholtz1868, cassirer1944concept, schmidt2016perception}. Equivariance is strongly related to the idea of \textit{symmetricity}. As these modifications do not modify the essence of the underlying object, they should be treated (and learned) as a single concept. Recently, several approaches have embraced these ideas to preserve symmetries including translations \cite{lecun1989backpropagation}, planar rotations \cite{dieleman2016, marcos2017rotation, HNet, weiler2018learning, DREN, cheng2018rotdcf, hoogeboom2018hexaconv, bekkers2018roto, veeling2018rotation, lenssen2018group, smets2020pde}, spherical rotations \cite{spherical, worrall2018cubenet, weiler20183d, thomas_tensor_2018, cohen2019gauge}, scaling \cite{marcos2018scale,worrall2019deep, sosnovik2020scaleequivariant} and general symmetry groups \cite{GCNN, kondor2018generalization, weiler2019general, cohen2019general, bekkers2020bspline,Romero2020Co-Attentive, Venkataraman2020Building}. 
\begin{figure}
    \centering
        \begin{subfigure}{0.22\textwidth}
        \includegraphics[width=\textwidth]{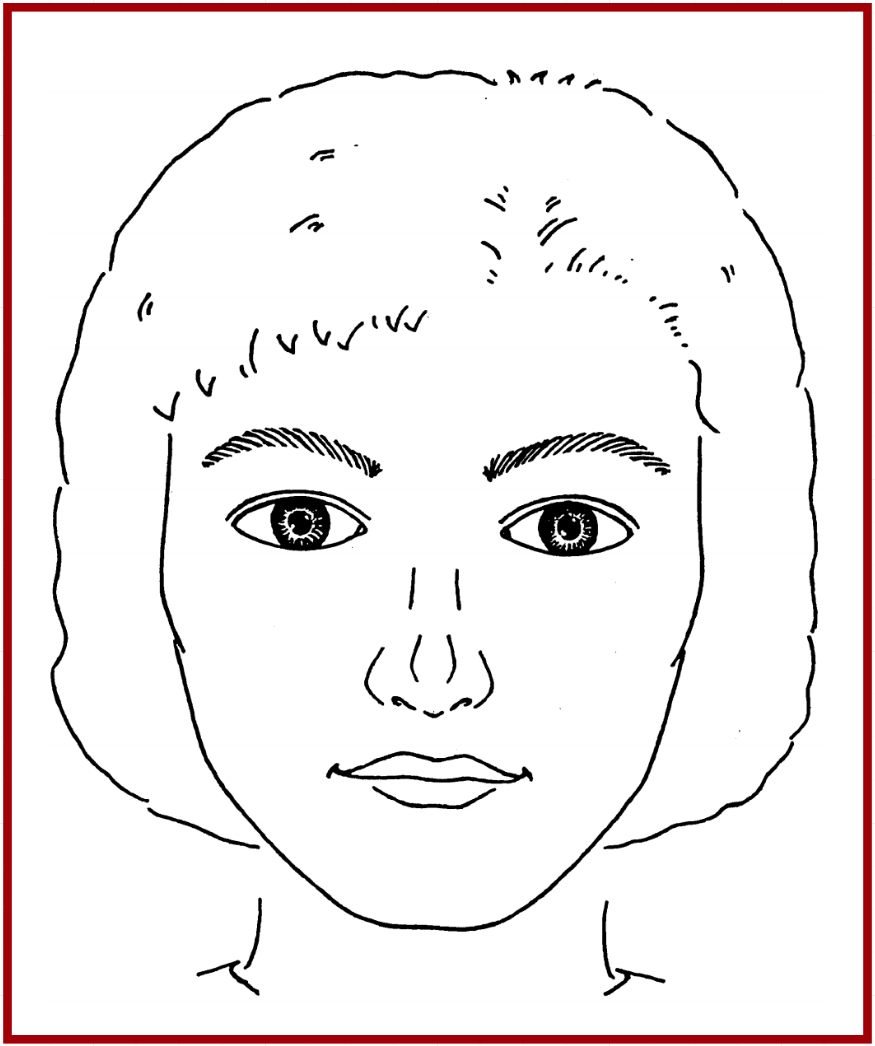}
    \end{subfigure}
        \begin{subfigure}{0.24\textwidth}
        \includegraphics[angle=90, width=\textwidth]{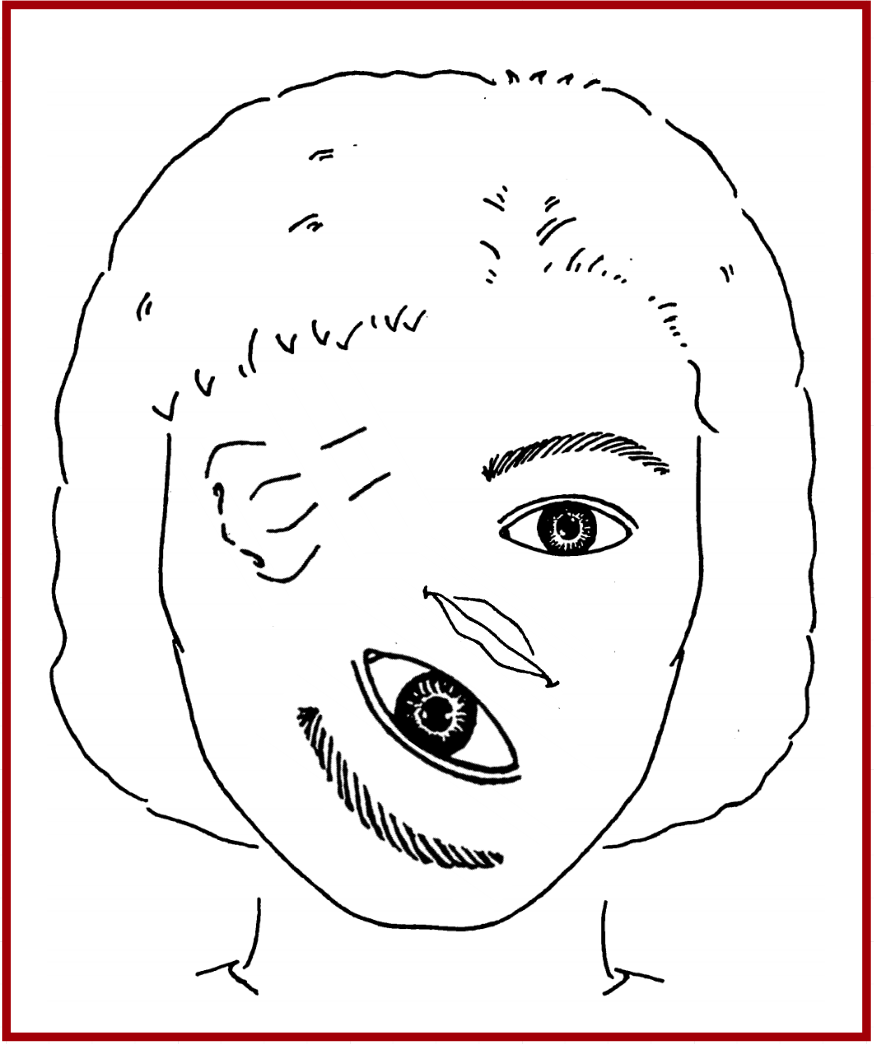}
    \end{subfigure}
    \begin{subfigure}{0.22\textwidth}
        \includegraphics[width=\textwidth]{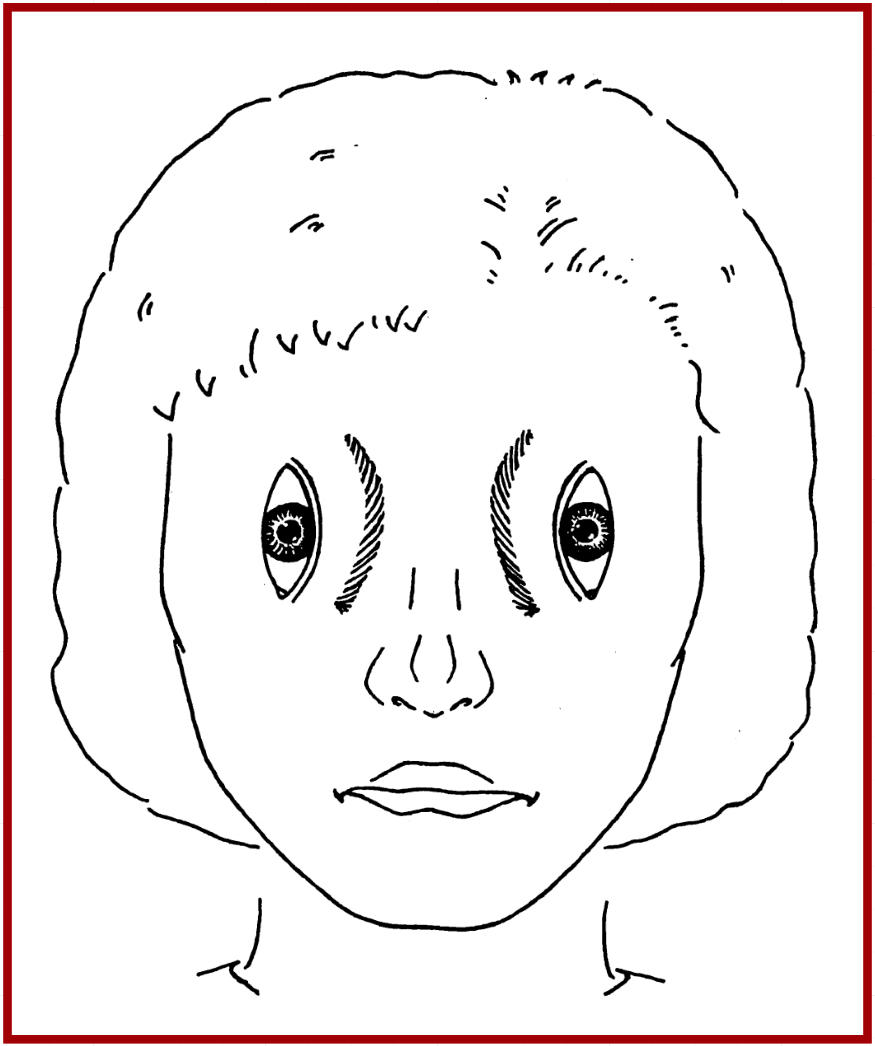}
    \end{subfigure}
        \begin{subfigure}{0.24\textwidth}
        \includegraphics[angle=-90, width=\textwidth]{images/good_face.png}
    \end{subfigure}
    \vskip -4mm
    \caption{Meaningful relationships among object symmetries. Though every figure is composed by the same elements, only the outermost examples resemble faces. The relative positions, orientations and scales of elements in the innermost examples do not match any meaningful face composition and hence, should not be labelled as such. Built upon Fig. 1 from \citet{schwarzer2000development}.}\label{fig:picasso}
\end{figure} 

While  group convolutional networks are able to learn powerful representations based on symmetry patterns, they lack any explicit means to learn meaningful relationships among them, e.g., relative positions, orientations and scales (Fig.~\ref{fig:picasso}). In this paper, we draw inspiration from another promising development in the machine learning domain driven by neuroscience and psychology (e.g., \citet{pashler2016attention}), \emph{attention}, to learn such relationships. The notion of attention is related to the idea that not all components of an input signal are \textit{per se} equally relevant for a particular task. As a consequence, given a task and a particular input signal, task-relevant components of the input should be focused during its analysis while irrelevant, possibly misleading ones should be suppressed. Attention has been broadly applied to fields ranging from natural language processing \cite{bahdanau2014neural, cheng2016long, vaswani2017attention} to visual understanding \cite{xu2015show, ilse2018attention, park2018bam, woo2018cbam, ramachandran2019stand, diaconu2019affine, Romero2020Co-Attentive} and graph analysis \cite{velivckovic2017graph, zhang2020hypersagnn}.

Specifically, we present \textit{attentive group convolutions}, a generalization of the group convolution, in which attention is applied during convolution to accentuate meaningful symmetry combinations and suppress non-plausible, possibly misleading ones. We indicate that prior work on visual attention can be described as special cases of our proposed framework and show empirically that our \textit{attentive group equivariant group convolutional networks} consistently outperform conventional group equivariant ones on rot-MNIST and CIFAR-10 for the $SE(2)$ and $E(2)$ groups. In addition, we provide means to interpret the learned concepts trough the visualization of the predicted equivariant attention maps.

\textbf{Contributions:}
\begin{itemize}[topsep=0pt, leftmargin=*]
\itemsep0em
  \item We propose a general group theoretical framework for equivariant visual attention, \textit{the attentive group convolution}, and show that prior works on visual attention are special cases of our framework.
    \item We introduce a specific type of network referred to as \textit{attentive group convolutional networks} as an instance of this theoretical framework. 
    \item We show that our \textit{attentive group convolutional networks} consistently outperform plain group equivariant ones.
  \item We provide means to interpret the learned concepts via visualization of the predicted equivariant attention maps.
\end{itemize}
\section{Preliminaries} \label{sec:prelim}
Before describing our approach, we first define crucial prior concepts: (group) convolutions and attention mechanisms.

\subsection{Spatial Convolution and Translation Equivariance}
Let $f$, $\psi:\mathbb{R}^{d}\rightarrow \mathbb{R}^{N_{\tilde{c}}}$ be a vector valued signal and filter on $\mathbb{R}^{d}$, such that $f=\{f_{\tilde{c}}\}_{\tilde{c}=1}^{N_{\tilde{c}}}$ and $\psi=\{\psi_{\tilde{c}}\}_{\tilde{c}=1}^{N_{\tilde{c}}}$. The spatial convolution ($\star_{\mathbb{R}^{d}}$) is defined as:
\begin{equation} \label{eq:norm_conv}
\setlength{\abovedisplayskip}{2pt}
    \setlength{\belowdisplayskip}{2pt}
    [f \star_{\mathbb{R}^{d}} \psi](y) = \sum_{\tilde{c}=1}^{N_{\tilde{c}}}\int_{\mathbb{R}^{d}}f_{\tilde{c}}(x)\psi_{\tilde{c}}(x - y) \, {\rm d}x
\end{equation}
Intuitively, Eq. \ref{eq:norm_conv} resembles a collection of $\mathbb{R}^{d}$ inner products between the input signal $f$ and $y$-translated versions of $\psi$.\break
Since the continuous integration in Eq. \ref{eq:norm_conv} is usually performed on signals and filters captured in a discrete grid $\mathbb{Z}^{d}$, the integral on $\mathbb{R}^{d}$ is reduced to a sum on $\mathbb{Z}^{d}$. In our derivations, however, we stick to the continuous case as to guarantee the validity of our theory for techniques defined on continuous spaces, e.g., steerable and Lie group convolutions\break\cite{cohen2016steerable, HNet, bekkers2018roto, weiler2018learning, weiler20183d, thomas_tensor_2018, weiler2019general, bekkers2020bspline, sosnovik2020scaleequivariant}.

To study (and generalize) the properties of the convolution, we rewrite Eq. \ref{eq:norm_conv} using the translation operator
$\mathcal{L}_{y}$: 
\begin{equation}\label{eq:norm_conv_goperator}
\setlength{\abovedisplayskip}{2pt}
    \setlength{\belowdisplayskip}{2pt}
    [f \star_{\mathbb{R}^{d}} \psi](y) = \sum_{\tilde{c}=1}^{N_{\tilde{c}}}\int_{\mathbb{R}^{d}}f_{\tilde{c}}(x)\mathcal{L}_{y}[\psi_{\tilde{c}}](x) \,{\rm d}x
\end{equation}
where $\mathcal{L}_{y}[\psi_{\tilde{c}}](x) = \psi_{\tilde{c}}(x-y)$. Note that the translation operator $\mathcal{L}_{y}$ is indexed by an amount of translation $y$. Resultantly, we actually consider a set of operators $\{\mathcal{L}_{y}\}_{y \in \mathbb{R}^{d}}$ that indexes the set of all possible translations $y \in \mathbb{R}^{d}$.\break A fundamental property of the convolution is that it commutes with translations:
\begin{equation} \label{eq:equiv}
    \mathcal{L}_{y}[f \star_{\mathbb{R}^{d}} \psi](x) = \big[\mathcal{L}_{y}[f] \star_{\mathbb{R}^{d}} \psi \big](x), \ \  x,y \in \mathbb{R}^{d}.
\end{equation}
In other words, convolving a $y$-translated signal $\mathcal{L}_{y}[f]$ with a filter is equivalent to first convolving the original signal $f$ with the filter $\psi$, and $y$-translating the obtained response next. This property is referred to as \textit{translation equivariance} and, in fact, convolution (and reparametrizations thereof) is the \textit{only} linear \textit{translation equivariant} mapping \cite{ kondor2018generalization, cohen2019general, bekkers2020bspline}.

\subsection{Group Convolution and Group Equivariance}
The convolution operation can be extended to general transformations by utilizing a larger set of transformations $\{\mathcal{L}_{g}\}_{g \in G}$, s.t. $\{\mathcal{L}_{y}\}_{y \in \mathbb{R}^{d}} \subseteq \{\mathcal{L}_{g}\}_{g \in G}$. However, in order to\break preserve equivariance, we must restrict the class of transformations allowed in $\{\mathcal{L}_{g}\}_{g \in G}$. To formalize this intuition, we first present some important concepts from \textit{group theory}.

\subsubsection{Preliminaries from Group Theory} \label{sec:group_theory}
\textbf{Groups.} A \textit{group} is a tuple ($G$, $\cdot$) consisting of a set $G$, $g \in G$, and a binary operation $\cdot: G\times G \rightarrow G$, referred to as the \textit{group product}, that satisfies the following axioms: 
\begin{itemize} [topsep=0pt, leftmargin=*]
\itemsep0em
    \item \textit{Closure:} For all $h$, $g \in G$, $h \cdot g \in G$. 
    \item \textit{Identity:} There exists an $e \in G$, such that $e\cdot g = g \cdot e = g$.
    \item \textit{Inverse:} For all $g \in G$, there exists an element $g^{-1} \in G$, such that $g \cdot g^{-1} = g^{-1} \cdot g = e$.
    \item \textit{Associativity:} For all $g,h,k \in G$, $(g\cdot h) \cdot k = g \cdot (h \cdot k)$.
\end{itemize}
\textbf{Group actions.} Let $G$ and $X$ be a group and a set, respectively. The (left) \textit{group action} of $G$ on $X$ is a function $\odot: G \times X \rightarrow X$ that satisfies the following axioms:
\begin{itemize} [topsep=0pt, leftmargin=*]
\itemsep0em
    \item \textit{Identity:} If $e$ is the identity of $G$, then, for any $x \in X$, $e \odot x=x$. 
    \item \textit{Compatibility:} For all $g$, $h \in G$, $x \in X$, $g \odot (h \odot x) = (g \cdot h)\odot x$.
\end{itemize}
In other words, the action of $G$ on $X$ describes how the elements $x\in X$ are transformed by $g \in G$. 
For brevity, we omit the operations $\cdot$ and $\odot$ and refer to the set $G$ as a group, to elements $g \cdot h$ as $gh$ and to actions $(g \odot x)$ as $gx$. 

\textbf{Semi-direct product and affine groups.} In practice, one is mainly interested in the analysis of data (and hence convolutions) defined on $\mathbb{R}^{d}$. Consequently, groups of the form $G = \mathbb{R}^{d} \rtimes H$, resulting from the \textit{semi-direct product} ($\rtimes$) between the translation group $\mathbb{R}^{d}$ and an arbitrary (Lie) group $H$ that acts on $\mathbb{R}^{d}$ (e.g., rotation, scaling, mirroring), are of main interest. This family of groups is referred to as \textit{affine groups} and their group product is defined as:
\begin{equation}\label{eq:semidirect}
    g_{1}g_{2}=(x_{1},h_{1})(x_{2},h_{2})=(x_{1}+h_{1}x_{2}, h_{1} h_{2})
\end{equation}
where $g_{1}=(x_{1}, h_{1})$, $g_{2}=(x_{2}, h_{2}) \in G$, $x_{1}, x_{2} \in \mathbb{R}^{d}$ and $h_1, h_2 \in H$. Some important affine groups are the roto-translation ($SE(d) = \mathbb{R}^{d} \rtimes SO(d)$), the scale-translation ($\mathbb{R}^{d} \rtimes \mathbb{R}^{+}$) and the euclidean ($E(d)=\mathbb{R}^{d} \rtimes O(d)$) groups.

\textbf{Group representations.} Let $G$ be a group and $\mathbb{L}_{2}(X)$ be a space of functions defined on some vector space $X$. The (left) regular \textit{group representation} of $G$ on functions $f \in \mathbb{L}_{2}(X)$ is a transformation  $\mathcal{L}: G \times \mathbb{L}_{2}(X) \rightarrow \mathbb{L}_{2}(X)$, $(g,f) \mapsto \mathcal{L}_{g}[f]$, such that it shares the group structure via:
\begin{gather}
    \mathcal{L}_{g}\mathcal{L}_{h}[f](x) = \mathcal{L}_{gh}[f](x) \\
    \mathcal{L}_{g}[f](x) 
    := f(g^{-1}x) 
\end{gather}
for any $g, h \in G$, $f \in \mathbb{L}_{2}(X)$, $x \in X$. That is, concatenating two such transformations, parametrized by $g$ and $h$, is equivalent to one transformation parametrized by $gh \in G$.
Intuitively, the representation of $G$ on a function $f \in \mathbb{L}_{2}(X)$ describes how the function as a whole, i.e., $f(x)$, $\forall \ x \in X$, is transformed by the effect of group elements $g \in G$.

If the group $G$ is affine, i.e., $G = \mathbb{R}^{d} \rtimes H$, the (left) group representation $\mathcal{L}_{g}$ can be split as:
\begin{equation} \label{eq:repr_decomp}
\mathcal{L}_{g}[f](x) = \mathcal{L}_{y} \mathcal{L}_{h}[f](x)
\end{equation}
with $g = (y, h) \in G$, $y \in \mathbb{R}^{d}$ and $h \in H$. This property is key for the efficient implementation of functions on groups.

\subsubsection{The Group Convolution} 
Let $f$, $\psi:G\rightarrow \mathbb{R}^{N_{\tilde{c}}}$ be a vector valued signal and kernel on $G$. 
The group convolution ($\star_{G}$) is defined as: 
\begin{align}\label{eq:g_conv}
    [f\star_{G} \psi ](g)
    &=\sum_{\tilde{c}=1}^{N_{\tilde{c}}}\int_{G}f_{\tilde{c}}(\tilde{g})\psi_{\tilde{c}}(g^{-1}\tilde{g}) \,{\rm d}\tilde{g}\\ \label{eq:g_conv2}
    &=\sum_{\tilde{c}=1}^{N_{\tilde{c}}}\int_{G}f_{\tilde{c}}(\tilde{g})\mathcal{L}_{g}[\psi_{\tilde{c}}](\tilde{g}) \,{\rm d}\tilde{g}
\end{align} 
\begin{figure}[t]
    \centering
    \includegraphics[width=0.94\textwidth]{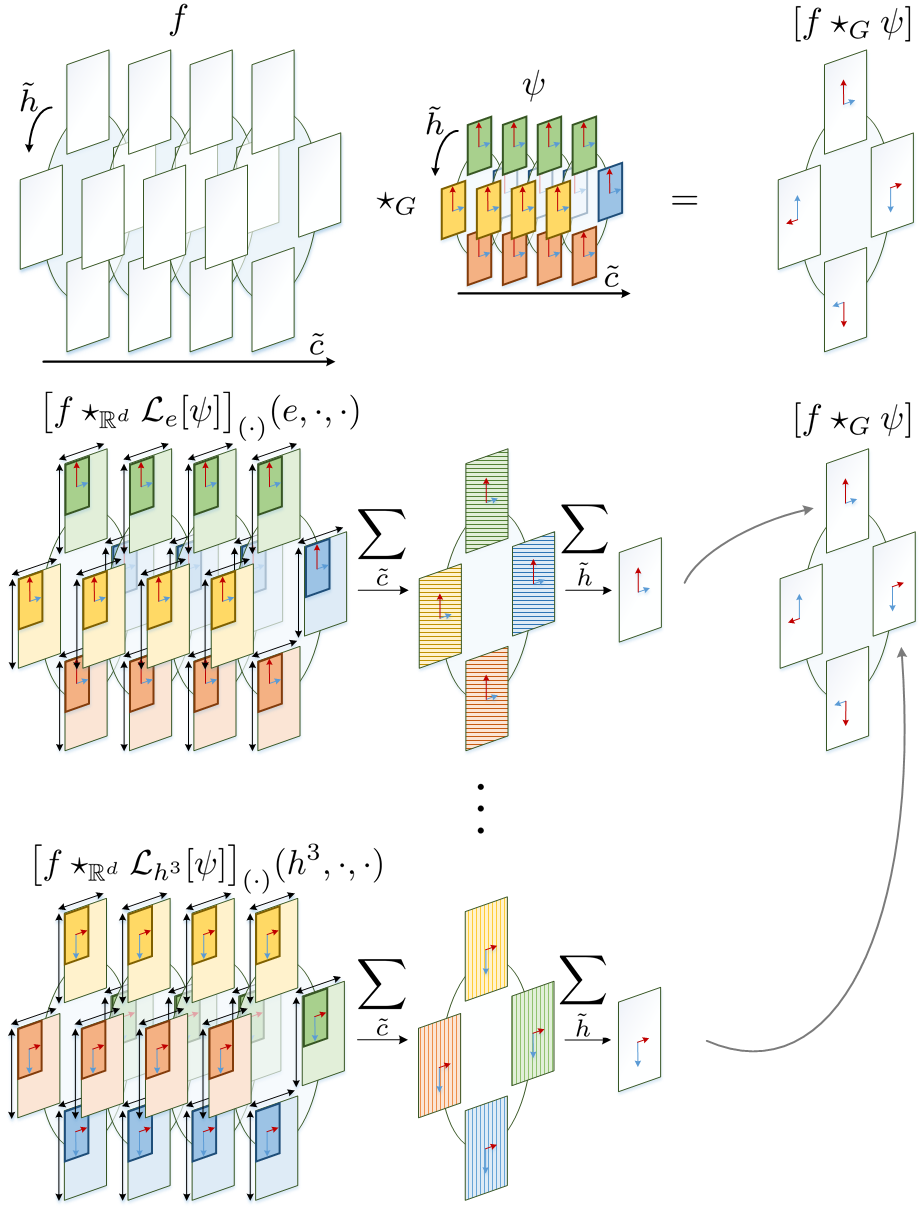}
    \vskip -4mm
    \caption{Group convolution on the roto-translation group $SE(2)$ for discrete rotations by 90 degrees (also called the $p4$ group). The $p4$ group is defined as $H = \{e, h, h^{2}, h^{3}\}$, with $h$ depicting a 90$^{\circ}$\break rotation. The group convolution corresponds to $|H|=4$ convolutions between the input $f$ and $h$-transformations of the filter $\psi$, $\mathcal{L}_{h}[\psi]$, $\forall \ h \in H$. Each of these convolutions is equal to the sum over group elements $\tilde{h} \in H$ and channels $\tilde{c} \in [N_{\tilde{c}}]$ of the spatial channel-wise convolutions $\big[f_{\tilde{c}} \star_{\mathbb{R}^{2}} \mathcal{L}_{h}[\psi_{\tilde{c}}]\big]$ among $f$ and $\mathcal{L}_{h}[\psi]$.}
    \label{fig:affine_conv}
\end{figure}
Differently to Eq. \ref{eq:norm_conv_goperator}, the domain of the signal $f$, the filter $\psi$ and the group convolution itself $[f \star_{G} \psi]$ are now defined on the group $G$.\footnote{Note that Eq. \ref{eq:norm_conv_goperator} matches Eq. \ref{eq:g_conv2} with the substitution $G=\mathbb{R}^{d}$. It follows that $\mathcal{L}_{g}[f](x)=f(g^{-1}x)=f(x-y)$, where $g^{-1}=-y$ is the inverse of $g$ in the translation group $(\mathbb{R}^{d}, +)$ for $g = y$.} Intuitively, the group convolution resembles a collection of inner products between the input signal $f$ and $g$-transformed versions of $\psi$.
A key property of the group convolution is that it generalizes equivariance (Eq. \ref{eq:equiv}) to arbitrary groups, i.e., it commutes with $g$-transformations:
\begin{equation}\label{eq:group_equiv}
        \mathcal{L}_{\bar{g}}[f \star_{G} \psi](g) = \big[\mathcal{L}_{\bar{g}}[f] \star_{G} \psi\big](g),\ \   g,\bar{g} \in G.
\end{equation}
In other words, group convolving a $\bar{g}$-transformed signal $\mathcal{L}_{\bar{g}}[f]$ with a filter $\psi$ is equivalent to first convolving the original signal $f$ with the filter $\psi$, and $\bar{g}$-transforming the obtained response next. This property is referred to as \textit{group equivariance} and, just as for spatial convolutions, the group convolution (or reparametrizations thereof) is the \textit{only} linear \textit{$G$-equivariant} map \cite{kondor2018generalization,cohen2019general,bekkers2020bspline}. 

\textbf{Group convolution on affine groups.} For affine groups, the group convolution (Eq. \ref{eq:g_conv2}) can be decomposed, without modifying its properties, by taking advantage of the group structure and the representation decomposition (Eq. \ref{eq:repr_decomp}) as:
\begin{align}
\setlength{\abovedisplayskip}{0pt}
    \setlength{\belowdisplayskip}{0pt}
    \setlength{\abovedisplayshortskip}{0pt}
\setlength{\belowdisplayshortskip}{0pt}
\hspace{-1mm}[f &\star_{G} \psi](g) =\sum_{\tilde{c}=1}^{N_{\tilde{c}}}  \int \limits_H \int \limits_{\mathbb{R}^{2}} f_{\tilde{c}}(\tilde{x},\tilde{h})\mathcal{L}_{g}[\psi_{\tilde{c}}](\tilde{x},\tilde{h}) \,{\rm d}\tilde{x}\,{\rm d}\tilde{h}\\[-1\jot] \label{eq:gconv_groupstruct_reprdecomp}
&\quad\quad\ \ \vspace{1mm}  =\sum_{\tilde{c}=1}^{N_{\tilde{c}}}\int \limits_H \int \limits_{\mathbb{R}^{2}} f_{\tilde{c}}(\tilde{x},\tilde{h})\mathcal{L}_{x}\mathcal{L}_{h}[\psi_{\tilde{c}}](\tilde{x},\tilde{h}) \,{\rm d}\tilde{x}\,{\rm d}\tilde{h}
\end{align}
where $g = (x,h)$, $\tilde{g}=(\tilde{x},\tilde{h}) \in G$, $x$, $\tilde{x}\in \mathbb{R}^{d}$ and $h$, $\tilde{h} \in H$. By doing so, the group convolution can be separated into $|H|$ spatial convolutions of the input signal $f$ for each $h$-transformed filter $\mathcal{L}_{h}[\psi]$ (Fig. \ref{fig:affine_conv}):
\begin{equation}
\setlength{\abovedisplayskip}{2pt}
    \setlength{\belowdisplayskip}{2pt}
 \label{eq:gconv_groupstruct_smallconvs}
[f \star_{G} \psi](x, h)=\sum_{\tilde{c}=1}^{N_{\tilde{c}}}\int_H \big[f_{\tilde{c}} \star_{\mathbb{R}^{2}} \mathcal{L}_{h}[\psi_{\tilde{c}}]\big](x, \tilde{h})\,{\rm d}\tilde{h}
\end{equation}
Resultantly, the computational cost of a group convolution is roughly equivalent
to that of a spatial convolution with a filter bank of size $N_{\tilde{c}}\times |H|$ \cite{GCNN, worrall2019deep, cohen2019gauge}.

\subsection{Attention, Self-Attention and Visual Attention} \label{sec:attention}
Attention mechanisms find their roots in recurrent neural network (RNN) based machine translation. Let  $\varphi(\cdot)$ be an arbitrary non-linear mapping (e.g., a neural network), $\underline{y}=\{y_{j}\}_{j=1}^{m}$ be a sequence of target vectors $y_i$, and $\underline{x}=\{x_{i}\}_{i=1}^{n}$ be a source sequence, whose elements influence the prediction of each value $y_{j} \in \underline{y}$. 
In early models (e.g., \citet{kalchbrenner2013recurrent, cho2014learning}), features in the input sequence are aggregated into a context vector $c = \sum_{i}\varphi(x_{i})$ which is used to augment the hidden state in RNN layers. These models assume that source elements $x_{i}$ contribute equally to \textit{every} target element $y_{j}$ and hence, that the same context vector $c$ can be utilized for all target positions $y_{j}$, which does not generally hold (Fig. \ref{fig:att}). 

\citet{bahdanau2014neural} proposed the inclusion of \textit{attention coefficients} $\alpha_{i}=\{\alpha_{i,j}\}$, $[n] = \{1, ..., n\}$, $i \in [n]$, $j \in [m]$, $\sum_{i}\alpha_{i,j}=1$, to modulate the contributions of the source elements $x_{i}$ as a function of the current target element $y_{j}$ by means of an adaptive context vector $c_{j} = \sum_{i}\alpha_{i, j}\varphi(x_{i})$. Thereby, they obtained large improvements both in performance and interpretability. 
Recently, attention has been extended to several other machine learning tasks (e.g., \citet{vaswani2017attention,velivckovic2017graph, park2018bam}). The main development behind these extensions was \textit{self-attention} \cite{cheng2016long}, where, in contrast to conventional attention, the target and source sequences are equal, i.e., $\underline{x}=\underline{y}$. Consequently, the attention coefficients $\alpha_{i,j}$ encode correlations among input element pairs $(x_{i}, x_{j})$.\break For vision tasks, self-attention has been proposed to encode visual co-occurrences in data \cite{hu2018squeeze, wang2018non, park2018bam, woo2018cbam, cao2019gcnet,bello2019attention, ramachandran2019stand, Romero2020Co-Attentive}. Unfortunately, its application on visual and, in general, on high-dimensional data is non-trivial.

\subsubsection{Visual Attention}\label{sec:visual_att}
\begin{figure}[t]
\floatbox[{\capbeside\thisfloatsetup{capbesideposition={right},capbesidewidth=3.8cm}}]{figure}[\FBwidth]
{\caption{English to French translation. Brighter depicts stronger influence. Note how relevant parts of the input sentence are highlighted as a function of the current output word during translation. Taken from \citet{bahdanau2014neural}.}\label{fig:att}}
{\includegraphics[width=3.2cm]{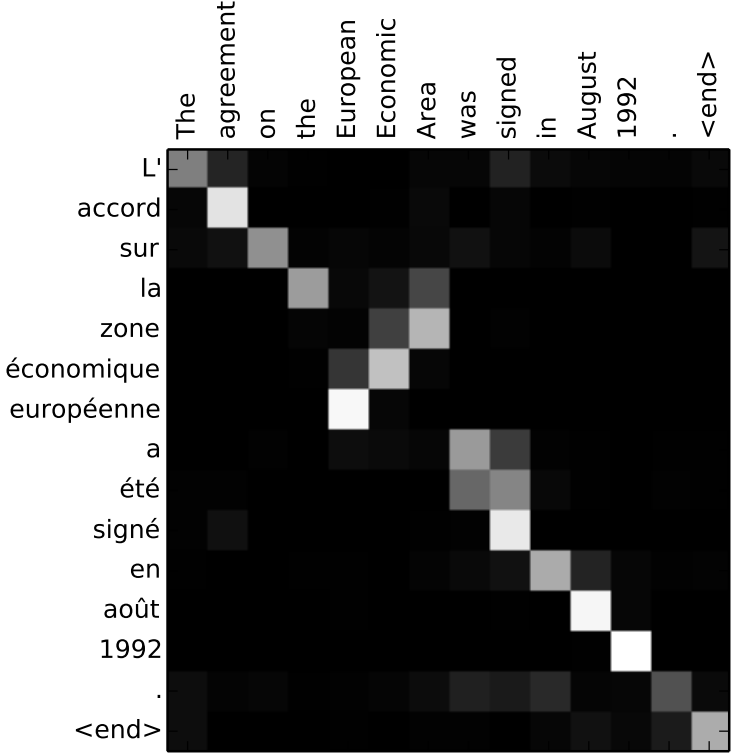}}
\end{figure}
In the context of visual attention, consider a feature map $f:{X}\rightarrow\mathbb{R}^{N_c}$ to be the source \enquote{sequence}\footnote{In the machine translation context we can think of $f$ as a sequence $\underline{x} = \{f(x_{i})\}_{i=1}^{n}$, with $n = |X|$ number of elements.}. Self-attention then imposes the learning of a total $n^{2} = |X|^2$ attention vectors $\alpha_{i,j} \in \mathbb{R}^{N_{\tilde{c}}}$, which rapidly becomes unfeasible with increasing feature map size.
Interestingly, \citet{cao2019gcnet} and \citet{zhu2019empirical} empirically demonstrated that, for visual data, the attention coefficients $\{\alpha_{i,j}\}$ are approximately invariant to changes in the target position $x_j$. 
Consequently, they proposed to approximate the attention coefficients $\{\alpha_{i,j}\} \in \mathbb{R}^{|X|^{2} \times N_{\tilde{c}}}$ by a single vector $\{\alpha_{i}\} \in \mathbb{R}^{|X| \times N_{\tilde{c}}}$ which is independent of target position $x_{j}$. 
Despite this significant reduction in complexity, the dimensionality of $\{\alpha_i\}$ is still very large and further simplifications are mandatory. To this end, existing works \cite{hu2018squeeze, woo2018cbam} replace the input $f$ with a much smaller vector of input\break statistics $s$ that summarizes relevant information from $f$.

For instance, the SE-Net \cite{hu2018squeeze} utilizes global average pooling to produce a vector of channel statistics of $f$, $s^{\mathcal{C}} \in \mathbb{R}^{N_{\tilde{c}}}$, $s^{\mathcal{C}} = \frac{1}{|\mathbb{R}^{d}|}\int_{\mathbb{R}^{d}}f_{\tilde{c}}(x)\,dx$, which is subsequently passed to a small fully-connected network $\varphi^{\mathcal{C}}(\cdot)$ to compute channel attention coefficients $\alpha^{\mathcal{C}} = \{\alpha^{\mathcal{C}}_{\tilde{c}}\}_{\tilde{c}=1}^{N_{\tilde{c}}} = \varphi^{\mathcal{C}}(s^{\mathcal{C}})$. These attention coefficients are then utilized to modulate the corresponding input channels $f_{\tilde{c}}$.

Complementary to channel attention akin to that of the SE-Net, \citet{park2018bam} utilize a similar strategy for spatial attention. Specifically, they utilize channel average pooling to generate a vector of spatial statistics of $f$, $s^{\mathcal{X}} \in \mathbb{R}^{d}$, $s^{\mathcal{X}}=\frac{1}{N_{\tilde{c}}}\sum_{\tilde{c}=1}^{N_{\tilde{c}}}f_{\tilde{c}}(x)$, which is subsequently passed to a small convolutional network $\varphi^{\mathcal{X}}(\cdot)$ to compute spatial attention coefficients $\alpha^{\mathcal{X}}=\{\alpha^{\mathcal{X}}(x)\}_{x \in \mathbb{R}^{2}} = \varphi^{\mathcal{X}}(s^{\mathcal{X}})$. These attention coefficients are then utilized to modulate the corresponding spatial input positions $f(x)$. 
Recent works include extra statistical information, e.g., max responses \cite{woo2018cbam}, or replace pooling by convolutions \cite{cao2019gcnet}.

\section{Attentive Group Equivariant Convolution}
In this section, we propose our generalization of visual self-attention, discuss its properties and relations to prior work. 
\begin{figure}
\floatbox[{\capbeside\thisfloatsetup{capbesideposition={right},capbesidewidth=6.4cm}}]{figure}[\FBwidth]
{\caption{Same colors depict equal weights. The first column of $\mathcal{A^{C}}$ corresponds to $\psi$ and the following ones to $\mathcal{L}_{h}[\psi]$, obtained via cyclic permutations. See how $\{\mathcal{L}_{h}[\psi]\}_{h \in H}$ resembles a circulant matrix. Taken from \citet{Romero2020Co-Attentive}.}\label{fig:circulant}}
{\includegraphics[width=1.35cm]{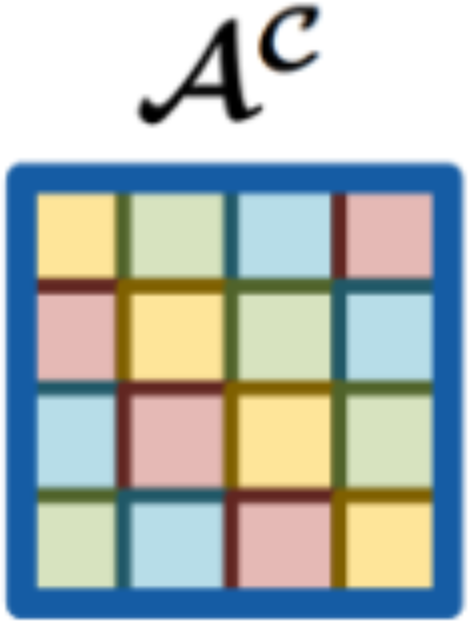}}
\end{figure}

Let $f, \psi: G \rightarrow \mathbb{R}^{N_{\tilde{c}}}$ be a vector valued signal and kernel on $G$, and let $\alpha: G \times G \rightarrow [0,1]^{N_{\tilde{c}}}$ be an \textit{attention map} that takes target and source elements $g,\tilde{g} \in G$, respectively, as input. We define the \textit{attentive group convolution} ($\star_{G}^{\alpha}$) as:
\begin{equation} \label{eq:att_g_conv}
    [f \star^{\alpha}_{G} \psi](g)
    =\sum_{\tilde{c}=1}^{N_{\tilde{c}}}\int_{G}\alpha_{\tilde{c}}(g,\tilde{g})f_{\tilde{c}}(\tilde{g})\mathcal{L}_{g}[\psi_{\tilde{c}}](\tilde{g}) \,{\rm d}\tilde{g}
\end{equation}
with $\alpha = \mathcal{A}[f]$ computed by some \textit{attention operator} $\mathcal{A}$. As such, the attentive group convolution modulates the contributions of group elements $\tilde{g} \in G$ at different channels $\tilde{c} \in [N_{\tilde{c}}]$ during pooling.\footnote{Note that Eq. \ref{eq:att_g_conv} is equal to Eq. \ref{eq:g_conv2} up to a multiplicative factor $\alpha_{\tilde{c}}(g,\tilde{g})^{-1}$, if $\alpha_{\tilde{c}}(g,\tilde{g})$ is constant for every $g,\tilde{g} \in G$, $\tilde{c} \in [N_{\tilde{c}}]$.} The properties and conditions on $\mathcal{A}$ are summarized in Thm.~\ref{thm1}. An extensive motivation as well as its proof are provided in the supplementary material.

\begin{theorem}
\label{thm1}
The attentive group convolution is an equivariant operator if and only if the attention operator $\mathcal{A}$ satisfies:
\begin{equation}
\label{eq:equivconstraint}
\forall_{\overline{g},g,\tilde{g} \in G}: \;\;  \mathcal{A}[\mathcal{L}_{\overline{g}}f](g,\tilde{g}) = \mathcal{A}[f](\overline{g}^{-1} g, \overline{g}^{-1} \tilde{g})
\end{equation}
If, moreover, the maps generated by $\mathcal{A}$ are invariant to one of its arguments, and, hence, exclusively attend to either the input or the output domain (Sec. \ref{sec:att_sequence}), then $\mathcal{A}$ satisfies Eq.~\ref{eq:equivconstraint} iff it is equivariant and thus, based on group convolutions.
\end{theorem}

\subsection{Tying Together Equivariance and Visual Attention} \label{sec:tyingtogether}
Interestingly, and, perhaps in some cases unaware of it, \textit{all} of the visual attention approaches outlined in Section \ref{sec:visual_att}, as well as all of those we are aware of \cite{xu2015show, hu2018squeeze, park2018bam, woo2018cbam, wang2018non, ilse2018attention, hu2019local, ramachandran2019stand, cao2019gcnet, chen2019graph, bello2019attention, lin2019contextgated, diaconu2019affine, Romero2020Co-Attentive} \textit{exclusively utilize translation (or group) equivariance preserving maps for the generation of the attention coefficients and, hence, constitute altogether group equivariant networks by which they satisfy Thm.~\ref{thm1}}.

As will be explained in the following sections, all these works resemble special cases of Eq.~\ref{eq:att_g_conv} by substituting $G$ with the corresponding group and modifying the specifications about how $\alpha$ is calculated (Sec. \ref{sec:equiv_attcoefs} - \ref{sec:att_sequence}).

\subsubsection{Translation Equivariant Visual Atention}
Since convolutions as well as popular pooling operations are translation equivariant, the visual attention approaches outlined in Sec. \ref{sec:visual_att} are translation equivariant as well.\footnote{In fact, conventional pooling operations (e.g., max, average) can be written as combinations of convolutions and pointwise non-linearities, which are translation equivariant, as well.} 
One particular case worth emphasising is that of SE-Nets. Here, a fully-connected network $\varphi^{\mathcal{C}}$, a non-translation equivariant map, is used to generate the channel attention coefficients $\alpha^{\mathcal{C}}$. However, $\varphi^{\mathcal{C}}$ \textit{is} indeed translation equivariant.\break Recall that $\varphi^{\mathcal{C}}$ receives $s^{\mathcal{C}}$ as input, a signal obtained via global average pooling (a convolution-like operation). Resultantly, $s^{\mathcal{C}}$ can be interpreted as a $\mathbb{R}^{N_{\tilde{c}} \times 1 \times 1}$ tensor and hence, applying a fully connected layer to $s^{\mathcal{C}}$ equals a pointwise convolution between $s^{\mathcal{C}}$ and a filter $\psi_{\text{fully}}\in \mathbb{R}^{N_o \times N_{\tilde{c}} \times 1 \times 1}$ with $N_o$ output channels.\footnote{This resembles a depth-wise separable convolution \cite{chollet2017xception} with the first convolution given by global average pooling.} 

\subsubsection{Group Equivariant Visual Attention}
To the best of our knowledge, the only work that provides a group theoretical approach towards visual attention  is that of \citet{Romero2020Co-Attentive}. Here, the authors consider affine groups $G$ with elements $g = (x,h)$, $x \in \mathbb{R}^{d}$, $h \in H$ and cyclic permutation groups $H$. Consequently, they utilize a cyclic permutation equivariant map, $\varphi^{\mathcal{H}}(\cdot)$, to generate attention coefficients $\alpha^{\mathcal{H}}(h)$, $h \in H$, with which the corresponding elements $h$ are modulated. As a result, their proposed attention strategy is $H$-equivariant. 
To preserve translation equivariance, and hence, $G$-equivariance, $\varphi^{\mathcal{H}}$ is re-utilized at every spatial position $x \in \mathbb{R}^{d}$. This is equivalent to combining $\varphi^{\mathcal{H}}$ with a pointwise filter on $\mathbb{R}^{d}$.
\citet{Romero2020Co-Attentive} found that equivariance to cyclic groups $H$, can \textit{only} be achieved by constraining $\varphi^{\mathcal{H}}$ to have a \textit{circulant structure}. This is equivalent to a convolution with a filter $\psi$, whose group representations $\mathcal{L}_{h}$ induce cyclical permutations of itself (Fig. \ref{fig:circulant}) and hence, resembles a group convolution, by which Thm.~\ref{thm1} is satisfied.

The work of \citet{Romero2020Co-Attentive} exclusively performs attention on the $h$ component of the group elements $g = (x,h) \in G$ and is only defined for (block) cyclic groups. Consequently, it does not consider spatial relationships during attention (Fig. \ref{fig:picasso}) and is not applicable to general groups. Conversely, our proposed framework allows for simultaneous attention on both components of the group\break elements $ g = (x, h)$ in a $G$ equivariance preserving manner.

\subsection{Efficient Group Equivariant Attention Maps}\label{sec:equiv_attcoefs}
Attentive group convolutions impose the generation of an additional attention map $\alpha: G \times G \rightarrow [0,1]^{N_{\tilde{c}}}$, which is computationally demanding. To reduce this computational burden, we exploit the fact that visual data is defined on $\mathbb{R}^{d}$ and, hence, relevant groups are affine, to provide an efficient factorization of the attention map $\alpha$. 

In Sec.~\ref{sec:visual_att} we indicated that attention coefficients $\alpha$ can be\break equivariantly factorized into spatial and channel components. We build upon this idea and factorize attention via:
$$
\alpha_{\tilde{c}}(g,\tilde{g}):=\alpha^{\mathcal{X}}(({x},h),(\tilde{{x}},\tilde{h}))\alpha_{\tilde{c}}^{\mathcal{C}}(h,\tilde{h})
$$
where $\alpha^{\mathcal{X}}$ attends for spatial relations without considering channel characteristics and $\alpha^{\mathcal{C}}$ attends for patterns in the channel- and $H$-axis, but ignores spatial patterns. We thus factorize $\alpha$ into a \emph{spatial attention map} $\alpha^{\mathcal{X}}: G \times G \rightarrow [0,1]$ and a \emph{channel attention map} $\alpha^{\mathcal{C}}: H \times H \rightarrow [0,1]^{N_{\tilde{c}}}$.\break
Findings in literature have shown that, for visual data, attention maps are almost equivalent for different query positions and thus, only query-independent dependencies are learnt \cite{cao2019gcnet,zhu2019empirical}. Based on this observation, we further simplify $\alpha^{\mathcal{X}}$ to be invariant over spatial positions either at the input or output space. Since separate convolutional filters $\psi$ could possibly benefit from different attention maps, we omit spatial positions in the input space (see Sec.~\ref{sec:att_operator} for details). In other words, we replace $\alpha^{\mathcal{X}}(g,\tilde{g})$ with $\alpha^{\mathcal{X}}(g,\tilde{h})$, an spatial position invariant attention map over the input space: $\alpha^{\mathcal{X}}: G \times H \rightarrow [0,1]$.

Conveniently, attention coefficients of type $\alpha: \mathbb{R}^d \times H \rightarrow [0,1]^{N_{\tilde{c}}}$ can be interpreted as functions on $\mathbb{R}^d$ with pointwise visualizations $\tilde{x} \mapsto \alpha(\tilde{x},\tilde{h})$ for each $\tilde{x}\in \mathbb{R}^{d}$. Resultantly, we are able to aid the interpretability of the learned concepts and of the attended symmetries (e.g., Figs.~\ref{fig:examples}, \ref{fig:pcam_examples}, \ref{fig:good_examples}).

\subsubsection{The Attention Operator $\mathcal{A}$}\label{sec:att_operator}
Recall that the attention map $\alpha$ is computed via an attention operator $\mathcal{A}$. In the most general case, $\alpha$ and, hence $\mathcal{A}$, is a function of both the input signal $f$ and the filter $\psi$. In order to define $\mathcal{A}$ as such, we generalize the approach of \citet{woo2018cbam} such that: (1) equivariance to general symmetry groups is preserved and (2) the attention maps depend on the filter $\psi$ as well.
\begin{figure}[t]
    \centering
    \includegraphics[width=0.99\textwidth]{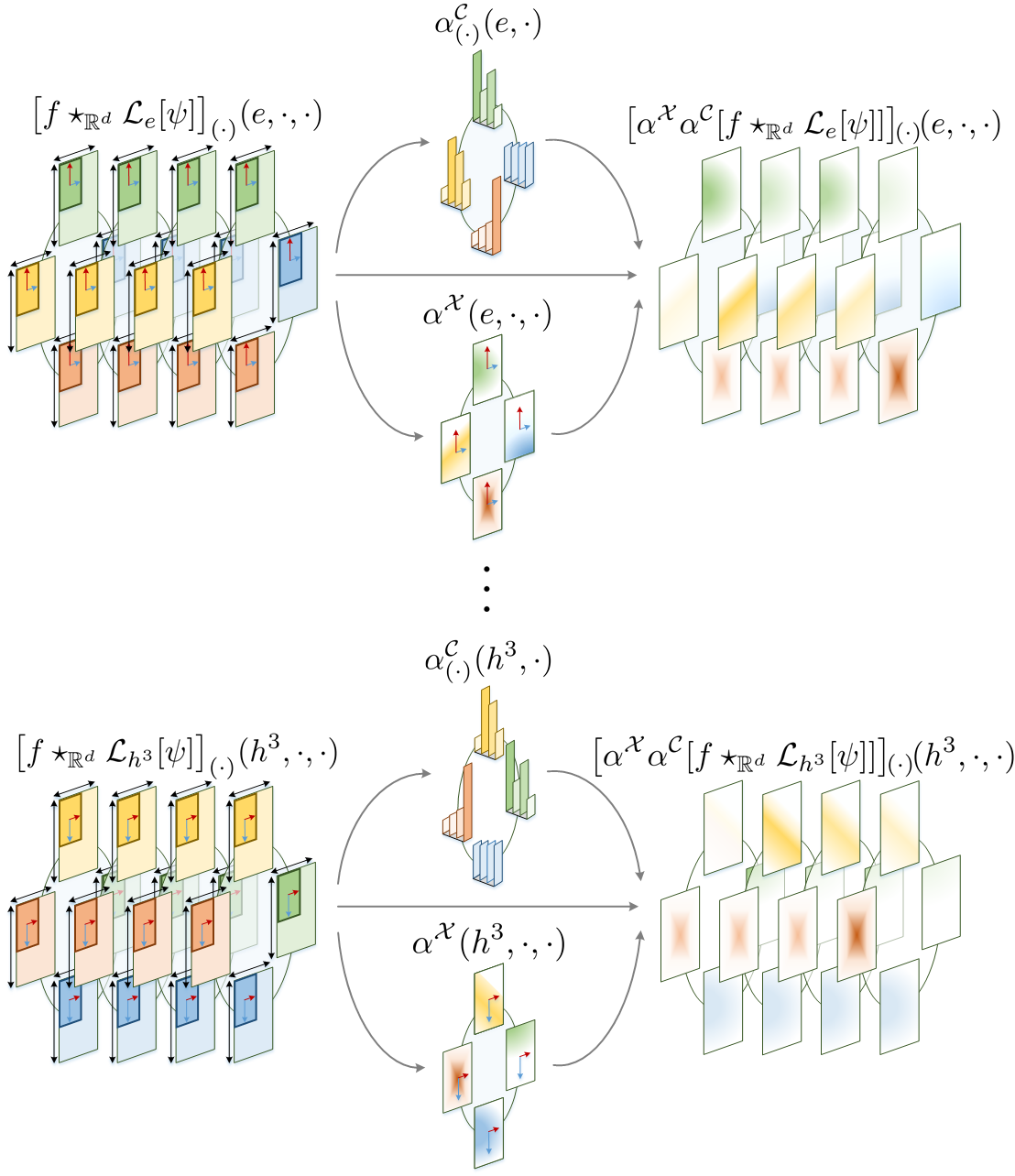}
    \vskip -4mm
    \caption{Attentive group convolution on the roto-translation group $SE(2)$. In contrast to group convolutions (Fig. \ref{fig:affine_conv}, Eq. \ref{eq:gconv_groupstruct_smallconvs}), attentive group convolutions utilize channel $\alpha^{\mathcal{C}}$ and spatial $\alpha^{\mathcal{X}}$ attention to modulate the intermediary convolutional responses $[f \star_{\mathbb{R}^{2}} \mathcal{L}_{h}[\psi]]$ before pooling over the $\tilde{c}$ and $\tilde{h}$ axes.}
    \label{fig:att_affine_conv}
\end{figure}

Let $\phi^{\mathcal{C}}: \tilde{f} \mapsto s^{\mathcal{C}} = \{s^{\mathcal{C}}_{\text{avg}}, s^{\mathcal{C}}_{\text{max}}\}$, $s^{\mathcal{C}}_{i}: H \times H \rightarrow \mathbb{R}^{N_{\tilde{c}}}$ and\break $\phi^{\mathcal{X}}: \tilde{f} \mapsto s^{\mathcal{X}} = \{s^{\mathcal{X}}_{\text{avg}}, s^{\mathcal{X}}_{\text{max}}\}$, $s^{\mathcal{X}}_{i}: G \times G \rightarrow \mathbb{R}$ be functions that generate channel ($s^{\mathcal{C}}$) and spatial statistics ($s^{\mathcal{X}}$),\break respectively, from an intermediary vector valued signal $\tilde{f}: G \times G \rightarrow \mathbb{R}^{N_{\tilde{c}}}$ containing information both from the input and output spaces. Analogously to \citet{woo2018cbam}, we compute spatial and channel statistics to reduce the dimensionality of the input. However, in contrast to them, we compute these statistics from intermediary convolutional maps $\tilde{f}$ rather than from the input signal $f$ directly.\footnote{This is why the statistics $s^{\mathcal{C}}_{i}$, $s^{\mathcal{X}}_{i}$ receive tuples $(h, \tilde{h})$, $(g, \tilde{g})$, respectively, as input, as opposed to single argument inputs which often emerge in several prior works on visual attention.} As a result,\break we take the influence of the filter $\psi$ into account during the computation of the attention maps. Following the simplifications proposed in Sec.~\ref{sec:equiv_attcoefs} for $\alpha^{\mathcal{X}}$, we can further reduce $s^{\mathcal{X}}_{i}$ and $\tilde{f}$ to functions of the form $s^{\mathcal{X}}_{i}: G \times H \rightarrow \mathbb{R}$ and $\tilde{f}: G \times H \rightarrow \mathbb{R}^{N_{\tilde{c}}}$, respectively. Consequently, we define:
\begin{equation}\label{eq:def_ftilde}
    \tilde{f} = \{\tilde{f}_{\tilde{c}}\}_{\tilde{c}=1}^{N_{\tilde{c}}}, \ \ \tilde{f}_{\tilde{c}}(x,h,\tilde{h}) := \big[f_{\tilde{c}} \star_{\mathbb{R}^{d}} \mathcal{L}_{h}[\psi_{\tilde{c}}]\big](x, \tilde{h}),
\end{equation} 
which is the intermediary result of the convolution between the input $f$ and the $h$-transformation of the filter $\psi$, $\mathcal{L}_{h}[\psi]$ before pooling over $\tilde{c}$ and $\tilde{h}$ (Fig. \ref{fig:att_affine_conv}, Eq. \ref{eq:gconv_groupstruct_smallconvs}).

\textbf{Channel Attention.} Let $\varphi^{\mathcal{C}}: s^{\mathcal{C}} \mapsto \alpha^{\mathcal{C}}$ be a function that generates a channel attention map $\alpha^{\mathcal{C}}: H \times H \rightarrow [0,1]^{N_{\tilde{c}}}$ from a vector of channel statistics $s^{\mathcal{C}}: H \times H \rightarrow \mathbb{R}^{N_{\tilde{c}}}$ of the intermediate representation $\tilde{f}$. Our channel attention computation is analogous to that of \citet{woo2018cbam} based on two fully connected layers. However, in our case, each linear layer is parametrized by a \textit{matrix-valued kernel} $\mathbf{W}_i:H \rightarrow \mathbb{R}^{N_{out} \times N_{in}}$, which we shift via left-regular representations $\mathcal{L}_h\left[\mathbf{W}_i\right](\tilde{h}) = \mathbf{W}_i(h^{-1} \tilde{h})$ in order to guarantee equivariance (Thm.~\ref{thm1}):
\begin{align}
    \alpha^{\mathcal{C}}(h,\tilde{h})&=\varphi^{\mathcal{C}}\left[s^{\mathcal{C}}\right](h,\tilde{h}) \label{eq:compute_ak_2}\\
    &\hspace{-1cm} = \sigma\Big(\big[
    \mathbf{W}_{2}(h^{-1}\tilde{h})\cdot[\mathbf{W}_{1}(h^{-1}\tilde{h})\cdot s^{\mathcal{C}}_{ \text{avg}}(h,\tilde{h})]^{+}
    \big] \nonumber  \\[-2\jot]
   &\hspace{0.45cm}
   + \big[
    \mathbf{W}_{2}(h^{-1}\tilde{h})\cdot[\mathbf{W}_{1}(h^{-1}\tilde{h})\cdot s^{\mathcal{C}}_{ \text{max}}(h,\tilde{h})]^{+}
    \big] \Big) \nonumber
\end{align}
with $[\cdot]^{+}$ the ReLU function, $\sigma$ the sigmoid function, $r$ a reduction ratio and $\mathbf{W}_{1}: H \rightarrow \mathbb{R}^{\frac{N_{\tilde{c}}}{r} \times N_{\tilde{c}}}$, $\mathbf{W}_{2}: H \rightarrow \mathbb{R}^{N_{\tilde{c}} \times \frac{N_{\tilde{c}}}{r}}$ filters defined on $H$.

\textbf{Spatial Attention.} Let $\varphi^{\mathcal{X}}: s^{\mathcal{X}} \mapsto \alpha^{\mathcal{X}}$ be a function that generates a spatial attention map $\alpha^{\mathcal{X}}: G \times H \rightarrow [0,1]$ from channel statistics $s^{\mathcal{X}}:G \times H \rightarrow \mathbb{R}^2$, in which per input $\tilde{h} \in H$ and output $g \in G$, the mean and max value is taken over the channel axis. Similarly to \citet{woo2018cbam}, spatial attention $\alpha^{\mathcal{X}}$ is then defined as:
\begin{align}
   \alpha^{\mathcal{X}}(x, h, \tilde{h}) &= \varphi^{\mathcal{X}}(s^{\mathcal{X}})(x, h, \tilde{h})\nonumber \\ \label{eq:compute_ax} &= \sigma\left(\big[s^{\mathcal{X}} \star_{\mathbb{R}^{d}} \mathcal{L}_{h}[ \psi^{\mathcal{X}}]\big]\right)(x, \tilde{h})
\end{align}
with $\psi^{\mathcal{X}}: G \rightarrow \mathbb{R}^{2}$ a group convolutional filter.

\textbf{Full Attention.} \citet{woo2018cbam} carried out extensive experiments to find the best performing configuration to combine channel and spatial attention maps for the $\mathbb{R}^{d}$ case, e.g., in parallel, serially starting with channel attention, serially starting with spatial attention. Based on their results we adopt their best performing configuration, i.e., \textit{serially starting with channel attention}, for the $G$ case (Fig. \ref{fig:attention_branch}). 

Recall that $\tilde{f}$ is the intermediary result from the convolution between the input $f$ and the $h$-transformation of the filter $\psi$ before pooling over $\tilde{c}$ and $\tilde{h}$. We perform attention on top of $\tilde{f}$ (Fig. \ref{fig:attention_branch}), where $\alpha^{\mathcal{C}}$ and $\alpha^{\mathcal{X}}$ are computed by Eqs. \ref{eq:compute_ak_2}, \ref{eq:compute_ax}, respectively. Resultantly, the attentive group convolution is computed as:
\begin{align}
[f \star^{\alpha}_{G} \psi](x, h)=
\sum_{\tilde{c}=1}^{N_{\tilde{c}}}\int_H \alpha^{\mathcal{X}}&(x, h, \tilde{h})\nonumber\\[-3\jot]
&\alpha_{\tilde{c}}^{\mathcal{C}}(h,\tilde{h})\tilde{f}(x,h,\tilde{h})\,{\rm d}\tilde{h} \label{eq:full_att}
\end{align}

\subsection{The Residual Attention Branch}\label{sec:res_att_branch}
Based on the findings of \citet{he2016deep}, several visual attention approaches propose to utilize residual blocks with direct connections during the course of attention to facilitate gradient flow \cite{hu2018squeeze, park2018bam, woo2018cbam, wang2018non, cao2019gcnet}. However, these approaches calculate the final attention map $\alpha^{+}$ as the sum of the direct connection $\boldsymbol{1}$ and the attention map obtained from the attention branch $\alpha$, i.e., $\alpha^{+}=\boldsymbol{1} + \alpha$. Consequently, the obtained attention map $\alpha^{+}: \mathbb{R}^{2} \rightarrow [1,2]^{N_{c}}$ is \textit{restricted} to the interval $[1, 2]$ and the network loses its ability to suppress input components. Inspired by the aforementioned works, we propose to calculate attention in what we call a \textit{residual attention branch} (Fig. \ref{fig:attention_branch}). Specifically, we utilize the attention branch to calculate a \textit{residual attention map} defined as $ \alpha^{-} = ( \boldsymbol{1} - \alpha^{+})$; $\alpha^{-}: G \times G \rightarrow [0,1]$. Next, we subtract the residual attention map $\alpha^{-}$ from the direct connection $\boldsymbol{1}$ to obtain the resultant attention map $\alpha^{+}$, i.e., $\alpha^{+} = \boldsymbol{1} - \alpha^{-}$. As a result, we are able to produce attention maps $\alpha^{+}$ that span the $[0,1]$ interval while preserving the benefits of the direct connections of \citet{he2016deep}.
\subsection{The Attentive Group Convolution as a Sequence of Group Convolutions and Pointwise Non-linearities}
\label{sec:att_sequence}
\begin{figure}[t]
    \centering
    \includegraphics[width=0.9\textwidth]{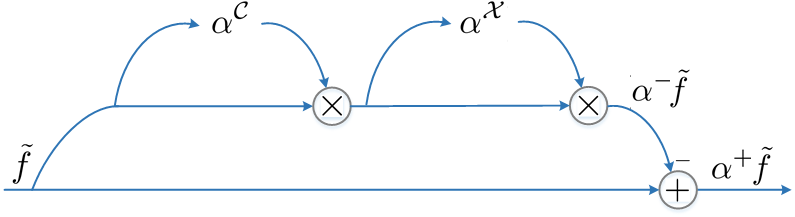}
    \vskip -6mm
    \caption{Sequential channel and spatial attention performed on a residual attention branch (Sec. \ref{sec:res_att_branch}).}
    \label{fig:attention_branch}
\end{figure}
CNNs are usually organized in layers and hence, the input $f$ is usually convolved in parallel with a set of $N_o$ filters $\{\psi_{o}\}_{o=1}^{N_o}$. As outlined in the previous section, this implies that the attention maps can change as a function of the current filter $\psi_{o}$. One assumption broadly utilized in visual attention is that these maps do not depend on the filters $\{\psi_{o}\}_{o=1}^{N_o}$, and, hence, that $\alpha$ is a sole function of the input signal $f$ \cite{hu2018squeeze, park2018bam, woo2018cbam, diaconu2019affine, Romero2020Co-Attentive}. Consequently, the attention coefficients $\alpha$ are reduced from a function $\alpha: G \times G \rightarrow [0,1]^{N_{\tilde{c}}}$ (c.f.,~Eq.~\ref{eq:att_g_conv}) to a function $\alpha: G \rightarrow [0,1]^{N_{\tilde{c}}}$. In other words, attention becomes only dependent on $g$ (see Eqs. \ref{eq:compute_ak_2}-\ref{eq:full_att}) and thus, the generation of the attention maps $\alpha^{\mathcal{C}}$, $\alpha^{\mathcal{X}}$ can be shifted to the input feature map $f$. 
Resultantly, the attentive group convolution is reduced to a sequence of conventional group convolutions and point-wise non-linearities (Thm.~\ref{thm1}), which further reduces the computational cost of attention:
\begin{equation}
\label{eq:sec_attention}
    [f \star^{\alpha}_{G} \psi] = [f^{\alpha} \star_{G} \psi] = [(\alpha^{\mathcal{X}}\alpha^{\mathcal{C}}f) \star_{G} \psi]
\end{equation}
\section{Experiments}
We validate our approach by exploring the effects of using attentive group convolutions in contrast to conventional ones.\break We compare the conventional group equivariant networks $p4$- and $p4m$-CNNs of \citet{GCNN} on the rotated MNIST and CIFAR-10 datasets with their corresponding attentive counterparts: $\alpha$-$p4$-CNNs and $\alpha$-$p4m$-CNNs, respectively; and the $p4$- and $p4m$-DenseNets of \citet{veeling2018rotation} on the PCam dataset with their corresponding attentive counterparts: $\alpha$-$p4$-DenseNet and $\alpha$-$p4m$- CNNs and DenseNets, respectively. Additionally, we explore the effects of only applying channel attention (e.g., $\alpha_{\text{CH}}$-$p4$-CNNs), spatial attention (e.g., $\alpha_{\text{SP}} $-$p4$-CNNs) and applying attention directly on the input (e.g., $\alpha_{F}$-$p4$-CNNs).\footnote{Our code is publicly available at:\\ \url{ https://github.com/dwromero/att_gconvs}}

We notice that the network architectures in \citet{GCNN} and \citet{Romero2020Co-Attentive} used for the CIFAR-10 experiments are equivariant only approximately. This results from using odd-sized convolutional kernels with stride $\geq$ 1 on even-sized feature maps (see Appx.~\ref{sec:approx_equiv} for a complete discussion). Since this effect distorts the equivariance property of our equivariant attention maps, i.e., they also become equivariant only approximately (Figs.~\ref{fig:bad_examples}, \ref{fig:good_examples}), this issue must be fixed. We achieve this by replacing strided convolutions in such regimes by conventional convolutions followed by a max-pooling layer.

For all our experiments we replicate as close as possible the training and evaluation strategies of the corresponding baselines, replace approximately equivariant networks by exact equivariant ones, and initialize any additional parameter in the same way as the corresponding baseline. Extended implementation details are provided in Appx.~\ref{appx:extended_details}.

\subsection{rot-MNIST}
The rotated MNIST dataset \cite{larochelle2007empirical} contains 62$k$ gray-scale 28x28 handwritten digits uniformly rotated for $[0, 2\pi)$. The dataset is split into training, validation and test sets of 10$k$, 2$k$ and 50$k$ images respectively. We compare $p4$-CNNs with all the corresponding attention variants previously mentioned. For our attention models, we utilize a filter size of $7$ and a reduction ratio $r$ of $2$ on the attention branch. Since attentive group convolutions impose the learning of additional parameters, we also instantiate bigger $p4$-CNNs by increasing the number of channels uniformly at every layer to roughly match the number of parameters of\break the attentive versions. Furthermore, we compare our results with comparative attentive versions as defined in \citet{Romero2020Co-Attentive} ($\alpha_{\text{RH}}$), which perform attention exclusively over the axis of rotations. Our results show that (1) attentive versions consistently outperform non-attentive ones, and that (2) performing attention over the entire group is beneficial in terms of classification accuracy (Tab. \ref{tab:rot_mnist}).

\subsection{CIFAR-10} The CIFAR-10 dataset \cite{krizhevsky2009learning} consists of 60$k$ real-world 32x32 RGB images uniformly drawn from 10 classes. The dataset is split into training, validation and test sets of $40k$, $10k$ and $10k$ images, respectively. We compare the $p4$ and $p4m$ versions of the All-CNN \cite{springenberg2014striving} and the Resnet44 \cite{he2016deep} in \citet{GCNN} with attentive variations. For all our attention models, we utilize a filter size of $7$ and a reduction ratio $r$ of $16$ on the attention branch. Unfortunately, attentive group convolutions impose an unfeasible increment on the memory requirements for this dataset.\footnote{the $\alpha$-$p4$ All-CNN requires approx. 72GB of CUDA memory, as opposed to 5GBs for the $p4$-All-CNN. This is due to the storage of the intermediary convolution responses required for the calculation of the attention weights (Eqs.~\ref{eq:compute_ak_2}- \ref{eq:full_att})} Resultantly, we are only able to compare the $\alpha_{\text{F}}$ variations of the corresponding networks. Our results show that attentive $\alpha_{\text{F}}$ networks consistently outperform non-attentive ones (Tab. \ref{tab:cifar}). Moreover, we demonstrate that our proposed networks focus on relevant parts of the input and that the predicted attention maps behave equivariantly for group symmetries (Figs.~\ref{fig:examples}, \ref{fig:good_examples}).

\begin{figure}
    \centering
    \hspace{0.25mm}
    \begin{subfigure}{0.3\textwidth}
        \includegraphics[width=\textwidth]{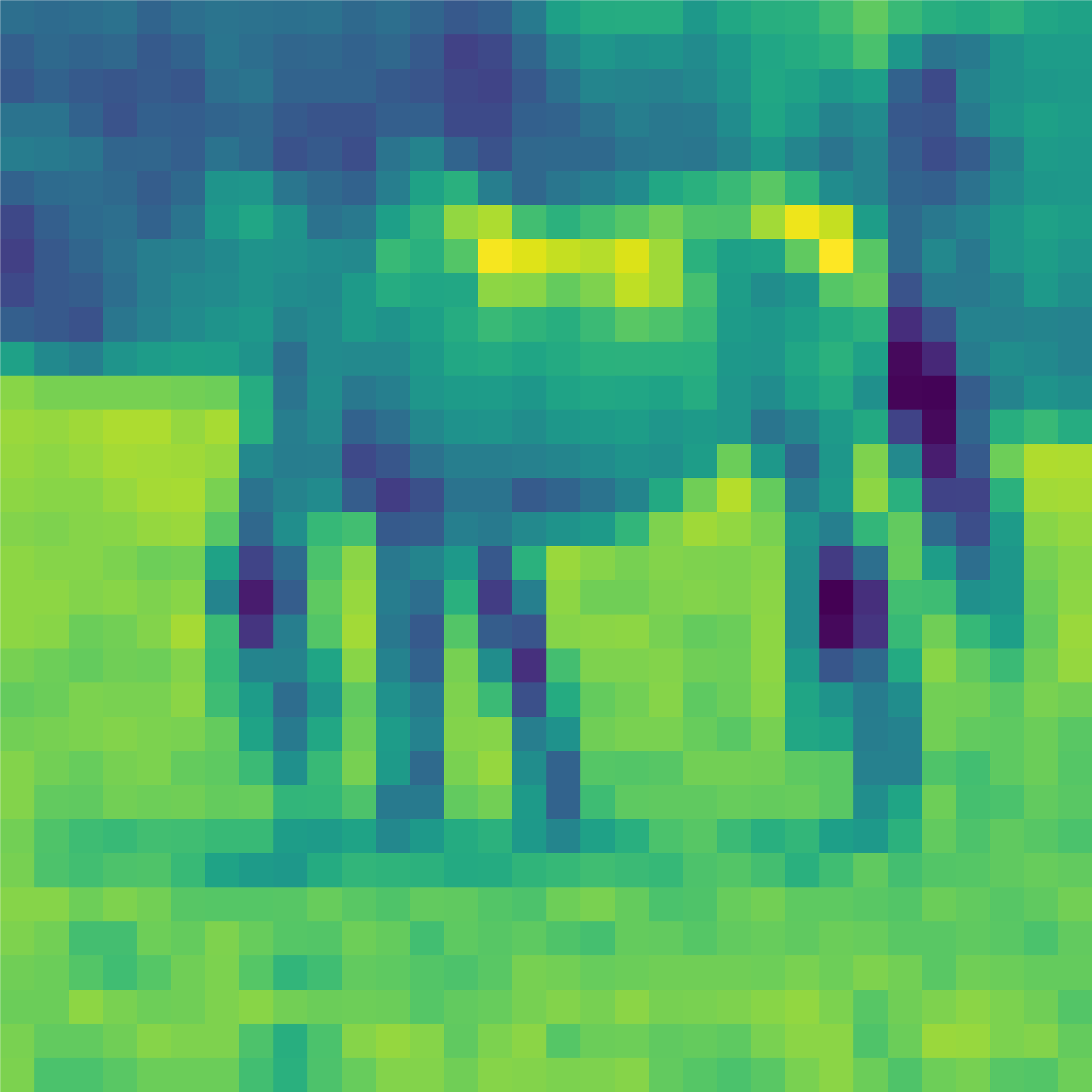}
    \end{subfigure}
    \quad
    \begin{subfigure}{0.3\textwidth}
        \includegraphics[width=\textwidth]{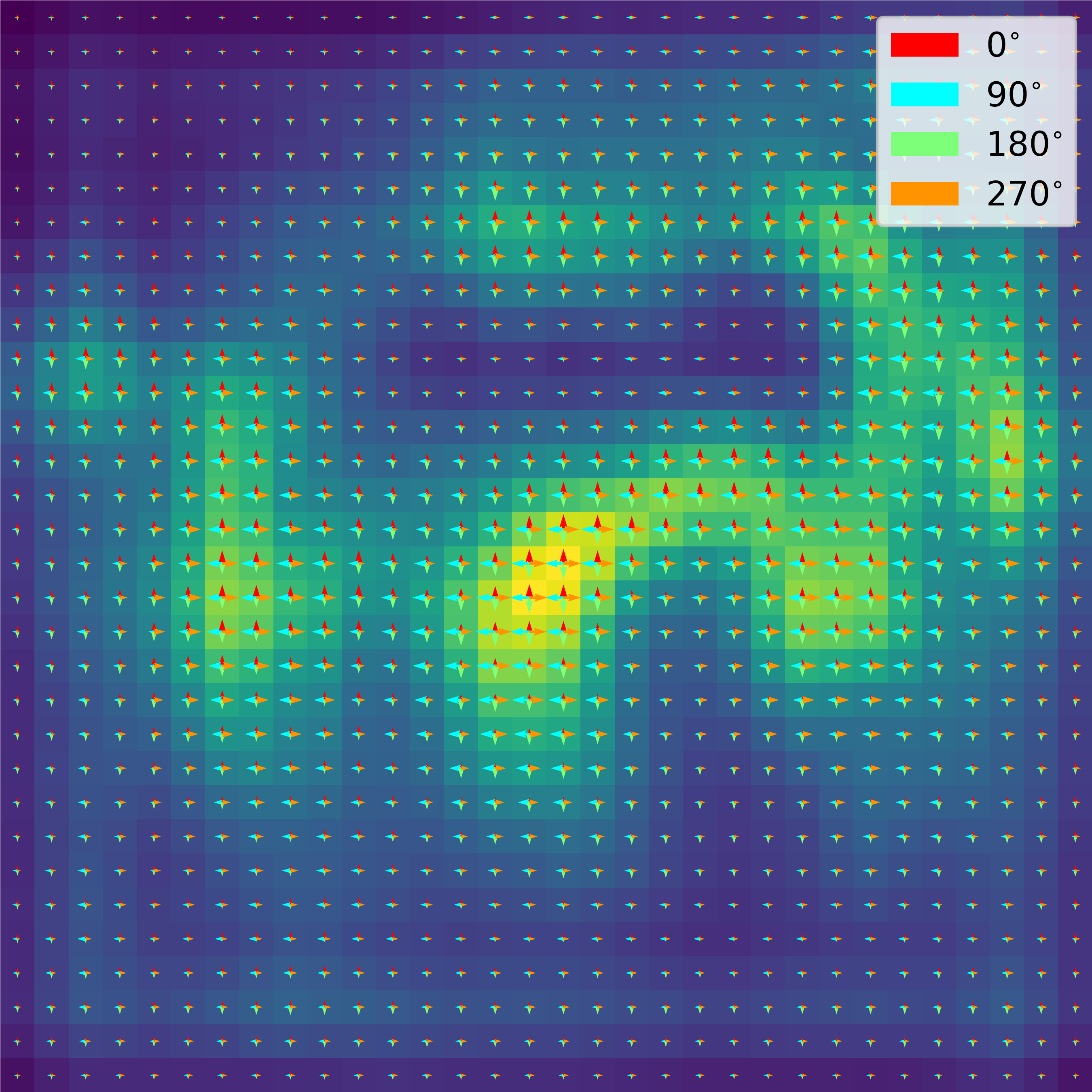}
    \end{subfigure}
    \vspace{0.3cm}

    \begin{subfigure}{0.3\textwidth}
        \includegraphics[width=\textwidth]{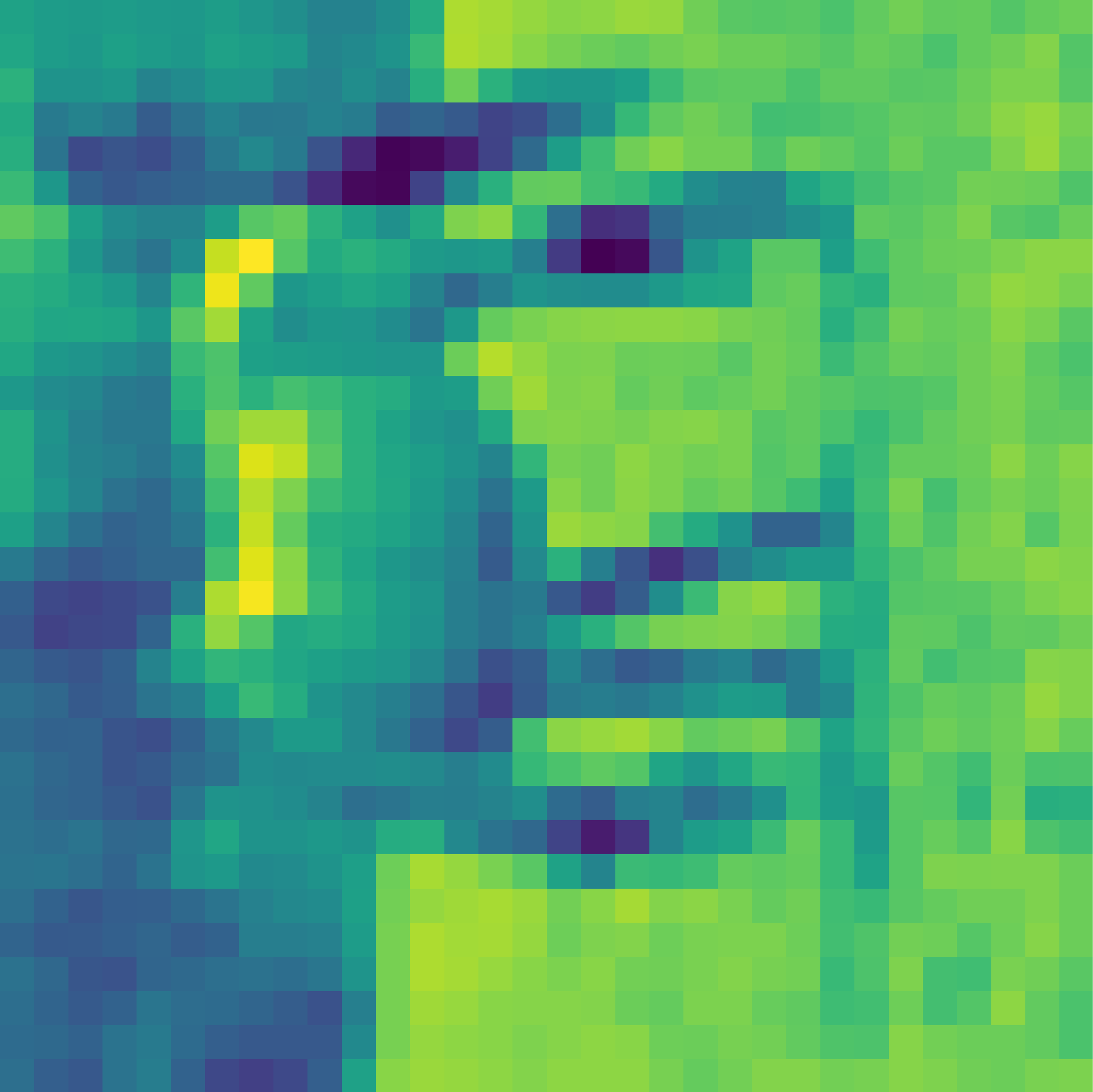}
    \end{subfigure}
    \quad
    \begin{subfigure}{0.3\textwidth}
        \includegraphics[width=\textwidth]{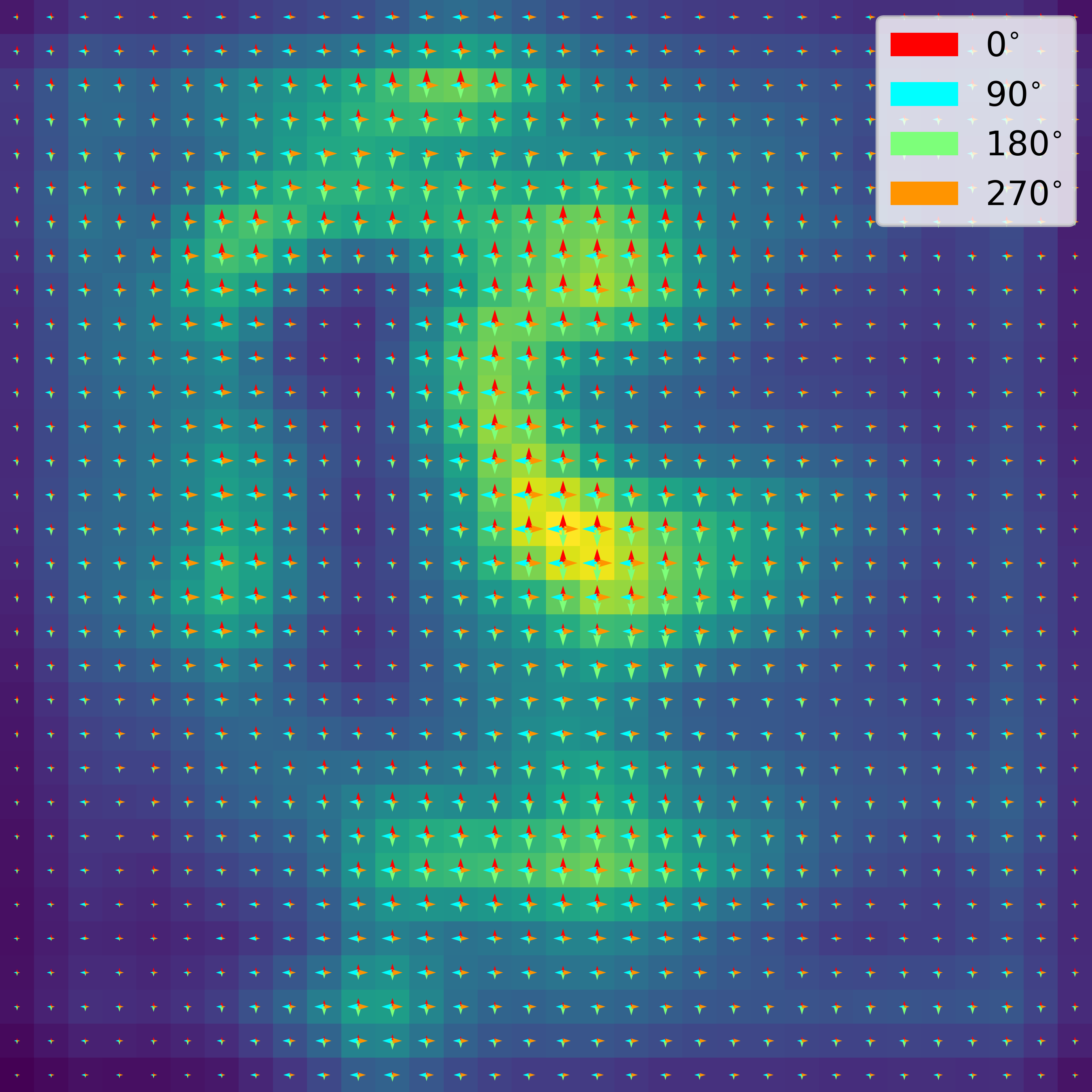}
    \end{subfigure}
    \vskip -4mm
    \caption{Equivariant attention maps on the roto-translation group $SE(2)$. The predicted attention maps behave equivariantly for group symmetries. The arrows depict the strength of the filter responses at the corresponding orientations throughout the network.}\label{fig:examples}
\end{figure}
\begin{table}[t]
\centering
\caption{Test error rates on rot-MNIST (with standard deviation under 5 random seed variations).}
\label{tab:rot_mnist}
\begin{center}
\vskip -3mm
\begin{small}
\begin{sc}
\scalebox{0.85}{
\begin{tabular}{rcc}
\toprule
Network & Test Error (\%) & Param.\\
\toprule
$p4$-CNN & 2.048 $\pm$ 0.045& 24.61k\\
$\alpha_{\text{RH}}$-$p4$-CNN & 1.980 $\pm$ 0.032 & 24.85k\\
\midrule
big$_{19}$-$p4$-CNN & 1.796 $\pm$ 0.035 & 77.54k\\
$\alpha$-$p4$-CNN & \textbf{1.696 $\pm$ 0.021} &  73.13k\\
\midrule
big$_{15}$-$p4$-CNN & 1.848 $\pm$ 0.019 & 50.42k\\
$\alpha_{\text{ch}}$-$p4$-CNN & \textbf{1.825 $\pm$ 0.048}& 48.63k \\
$\alpha_{\text{sp}}$-$p4$-CNN & \textbf{1.761 $\pm$ 0.027} & 49.11k \\
\midrule
big$_{11}$-$p4$-CNN & 1.996 $\pm$ 0.083  & 29.05k \\
$\alpha_{\text{f}}$-$p4$-CNN & \textbf{1.795 $\pm$ 0.028} & 29.46k \\
\bottomrule
\end{tabular}}
\end{sc}
\end{small}
\end{center}
\vskip -0.25in
\end{table}
\begin{table}[t!]
\caption{Test error rates on CIFAR10 and augmented CIFAR10+.}
\vskip -3mm
\label{tab:cifar}
\begin{center}
\begin{small}
\begin{sc}
\scalebox{0.85}{
\begin{tabular}{r|cccc}
\toprule
Network & Type & CIFAR10 & CIFAR10+ & Param.\\
\toprule
\multirow{4}{*}{All-CNN} & $p4$ & 9.32 & 8.91 & 1.37M\\
& $\alpha_{\text{f}}$-$p4$ &\textbf{ 8.8} & \textbf{7.05} & 1.40M\\
& $p4m$ & 7.61 & 7.48 & 1.22M\\
& $\alpha_{\text{f}}$-$p4m$ & \textbf{6.93} & \textbf{6.53} & 1.25M\\
\midrule
\multirow{2}{*}{ResNet44} & $p4m$ & 15.72& 15.4 & 2.62M\\
& $\alpha_{\text{f}}$-$p4m$ &\textbf{ 10.82 }& \textbf{10.12} & 2.70M\\
\bottomrule
\end{tabular}}
\end{sc}
\end{small}
\end{center}
\vskip -0.2in
\end{table}
\subsection{PCam} The PatchCamelyon dataset \cite{veeling2018rotation} consists of 327$k$ 96x96 RGB image patches of tumorous/non-tumorous breast tissues extracted from the Camelyon16 dataset \cite{bejnordi2017diagnostic}, where each patch was labelled as tumorous if the central region (32x32) contained at least one tumour pixel as given by the original annotation in \citet{bejnordi2017diagnostic}. We compare the $p4$ and $p4m$ versions of the DenseNet \cite{huang2017densely} in \citet{veeling2018rotation} with attentive variants. For all our attention models, we utilize a filter size of $7$ and a reduction ratio $r$ of 16 on the attention branch. Similarly to the CIFAR-10 case, we restrict our experiments to $\alpha_{\text{F}}$ attentive networks due to computational constraints. Our results show that attentive $\alpha_{\text{F}}$ consistently outperform non-attentive ones (Tab.~\ref{tab:pcam}). Interestingly, the $\alpha_{\text{F}}$-$p4$-DenseNet is already able to outperform the $p4m$-DenseNet without attention. Surprisingly, our equivariant attention maps reveal that the network learns to focus on the nuclei of the cells and to removes background elements during inference, all of this in a group equivariant way (Fig.~\ref{fig:pcam_examples}).

\begin{figure}
    \centering
    \hspace{0.25mm}
    \begin{subfigure}{0.3\textwidth}
        \includegraphics[width=\textwidth]{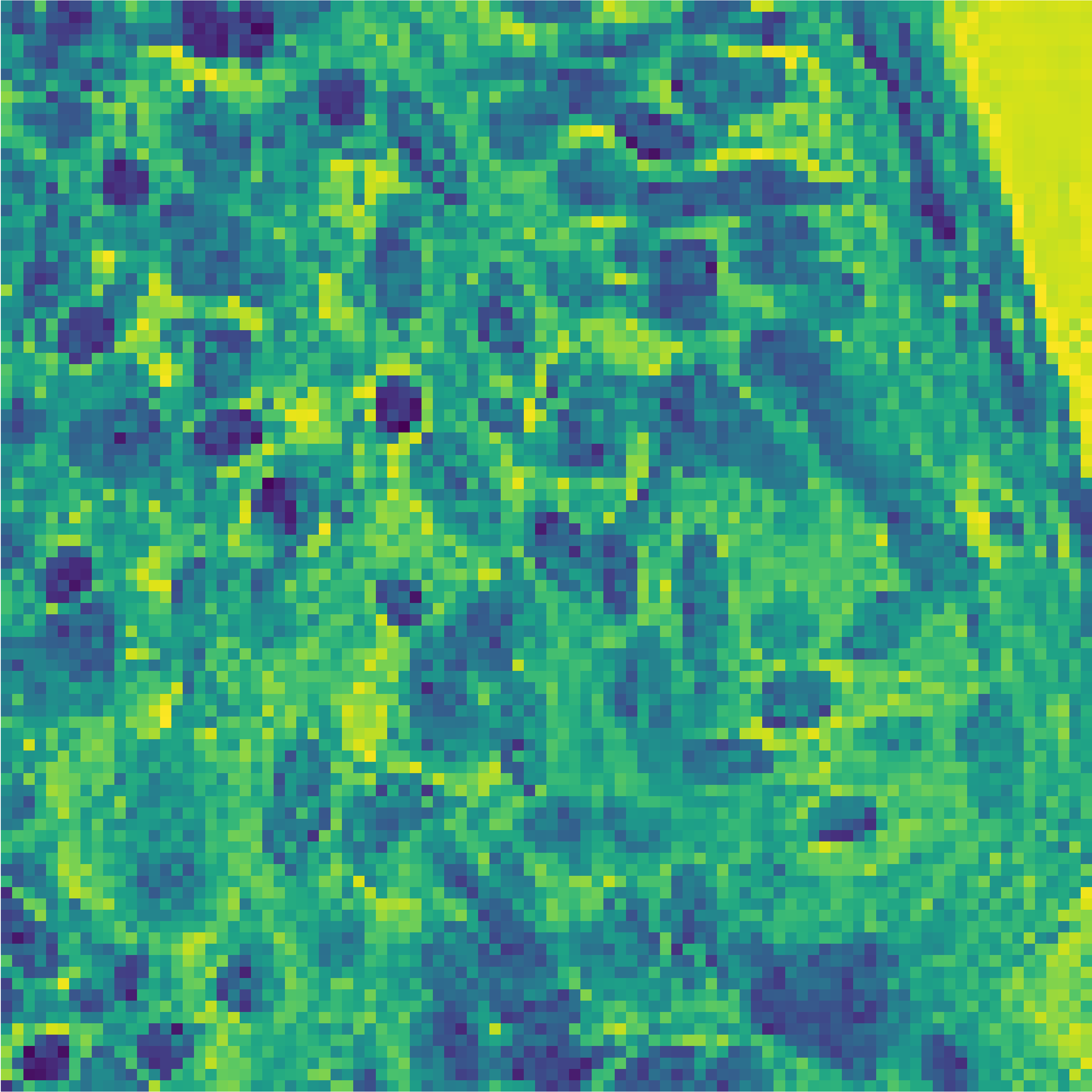}
    \end{subfigure}
    \quad
    \begin{subfigure}{0.3\textwidth}
        \includegraphics[width=\textwidth]{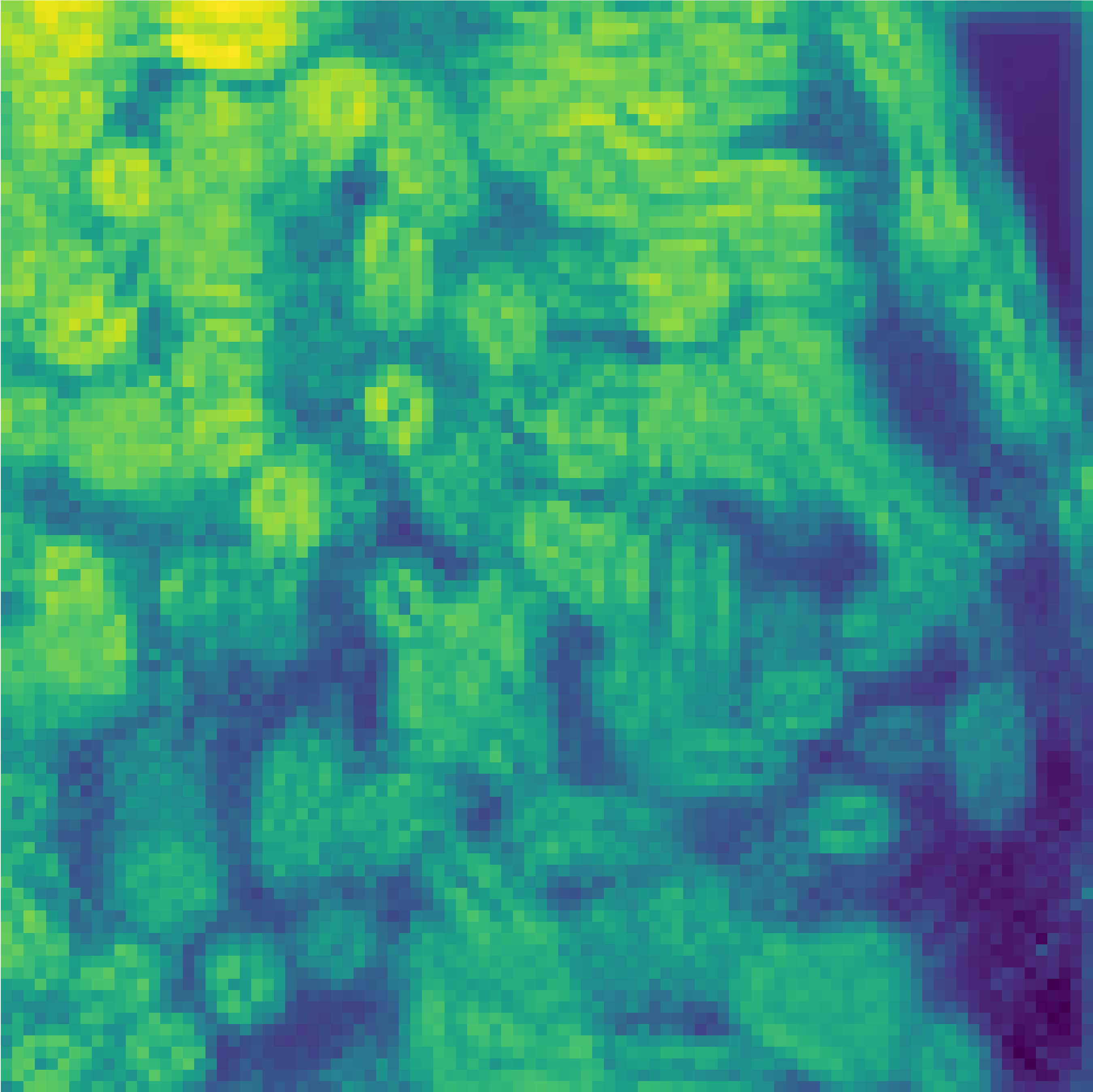}
    \end{subfigure}
    \vspace{0.3cm}

    \begin{subfigure}{0.3\textwidth}
        \includegraphics[width=\textwidth]{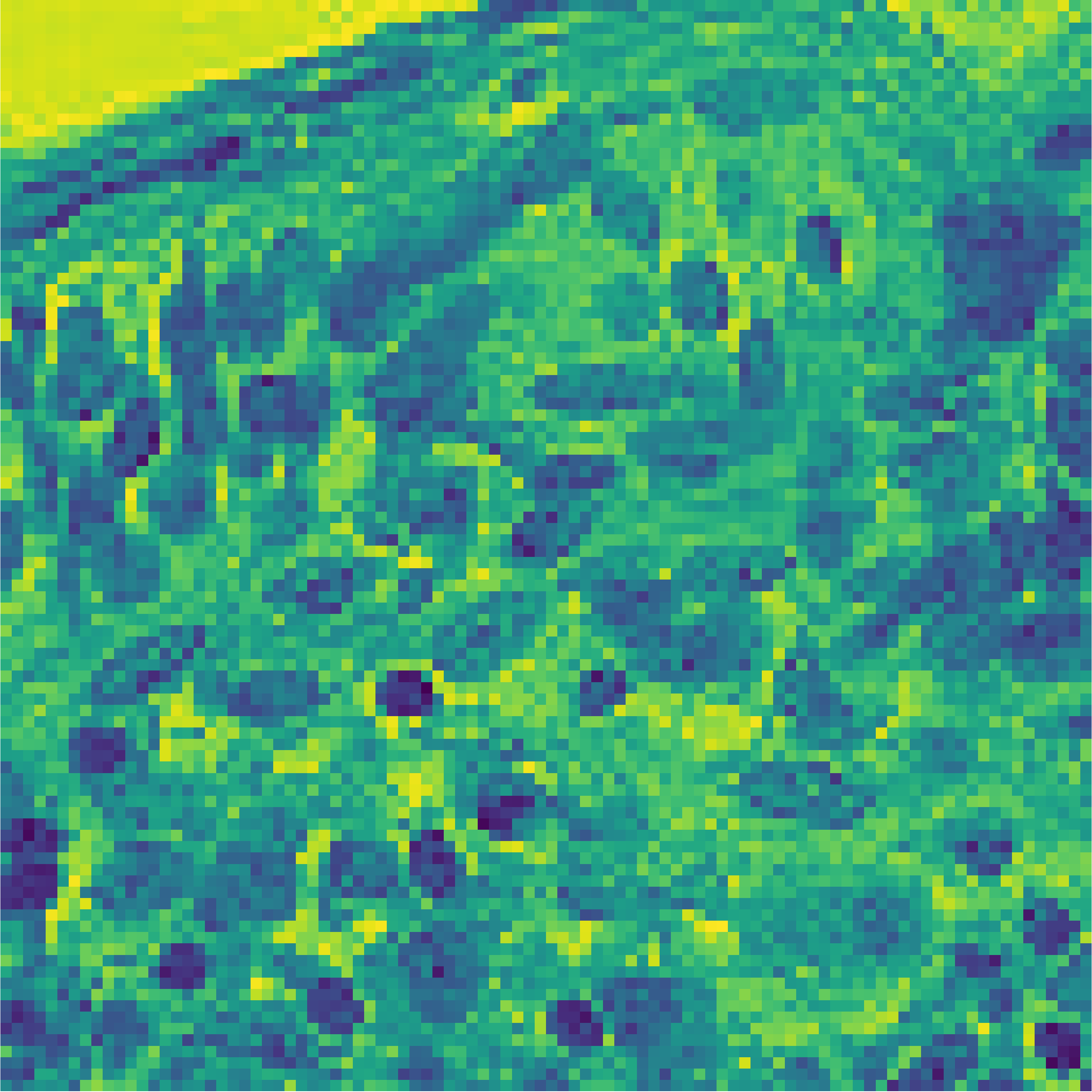}
    \end{subfigure}
    \quad
    \begin{subfigure}{0.3\textwidth}
        \includegraphics[width=\textwidth]{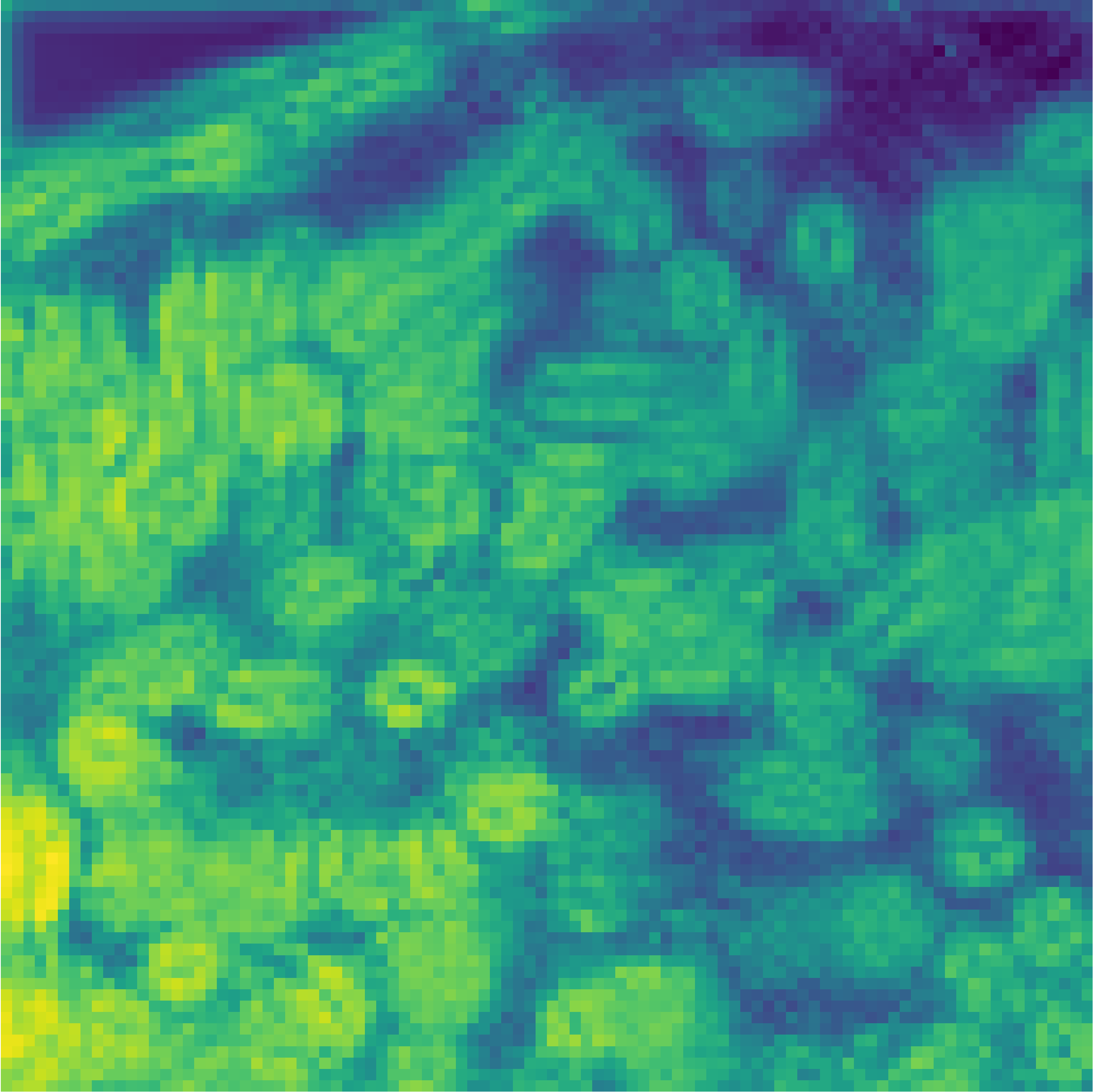}
    \end{subfigure}
    \vskip -4mm
    \caption{Equivariant attention maps on the PCam dataset. The predicted attention maps behave equivariantly for group symmetries. Additionally, the network seems to learn to focus on the nuclei of the cells and remove background elements during training.}\label{fig:pcam_examples}
\end{figure}
\begin{table}[t]
\caption{Test error rates on PCam.}
\vskip -3mm
\label{tab:pcam}
\begin{center}
\begin{small}
\begin{sc}
\scalebox{0.85}{
\begin{tabular}{r|cccc}
\toprule
Network & Type & Test Error (\%) & Param.\\
\toprule
\multirow{5}{*}{DenseNet}& $\mathbb{Z}^{2}$ & 15.93 &  130.60k\\
& $p4$ & 12.45 &  129.65k\\
& $\alpha_{\text{f}}$-$p4$ &\textbf{11.34} & 140.45k\\
& $p4m$ & 11.64  & 124.21k\\
& $\alpha_{\text{f}}$-$p4m$ & \textbf{10.88} &  141.22k\\
\bottomrule
\end{tabular}}
\end{sc}
\end{small}
\end{center}
\vskip -0.2in
\end{table}
\section{Discussion and Future Work}
Our results show that attentive group convolutions can be utilized as a drop-in replacement for standard and group equivariant convolutions that simultaneously facilitates the interpretability of the network decisions. Similarly to convolutional and group convolutional networks, attentive group convolutional networks also benefit of data augmentation. Interestingly, however, we also see that including additional symmetries reduces the effect of augmentations given by group elements. This finding supports the intuition that symmetry variants of the same concept are learned independently for non-equivariant networks (see Fig. 2 in \cite{alexnet}). The main shortcoming of our approach is its computational burden. As a result, the application of $\alpha$-networks is computationally unfeasible for networks with several layers or channels. We believe, however, by extrapolation of our results on rot-MNIST, that further performance improvements are to be expected for $\alpha$ variations, should hardware requirements suffice. 

Group convolutional networks have recently been proven very successful in medical imaging applications \cite{bekkers2018roto, winkels20183d, lafarge2020rototranslation}. Since explainability plays a crucial role here, we believe that our attentive maps could be of high relevance to aid the explainability of the network decisions. Moreover, since our attention maps are guaranteed to be equivariant to transformations in the considered group, it is ensured that the predicted attention maps will be consistent across group symmetries. We believe this to be of crucial importance for rotation invariant tasks. Illustratively, in contrast to vanilla attentive CNNs, a malignant tissue will be ensured to generate consistent attention maps regardless of the orientation at which it has been provided to the network.

In future work, we want to explore ways to reduce the computational cost of full attention networks. If successful, we consider feasible to obtain a direct performance boost over our CIFAR-10 and PCam experimental results, without extensive additional memory requirements. Furthermore, we want to extend our work to symmetry groups defined on 3D. By doing so, we expect the range of possible applications of our work to reach several other important applications such as 3D medical imaging applications like CT-scans and other voxel-based representations.

\section{Conclusion}
We introduced attentive group convolutions, a generalization of the group convolution in which attention is utilized to explicitly highlight meaningful relationships among symmetries. We provided a general mathematical framework for group equivariant visual attention and indicated that prior work on visual attention can be perfectly described as special cases of the attentive group convolution. Our experimental results indicate that attentive group convolutional networks consistently outperform conventional group convolutional ones and additionally provide equivariant attention maps that behave predictively for symmetries of the group, with which learned concepts can be visualized.

\section*{Acknowledgements}
We gratefully acknowledge our anonymous reviewers for their helpful and valuable commentaries, and Hyunjik Kim for valuable remarks to improve the readability of our paper. This work is part of the Efficient Deep Learning (EDL) programme (grant number P16-25), partly funded by the Dutch Research Council (NWO) and Semiotic Labs, and the research programme VENI (grant number 17290), financed by the Dutch Research Council (NWO). This work was carried
out on the Dutch national e-infrastructure with the support of SURF Cooperative.

\bibliography{example_paper}
\bibliographystyle{icml2020}

\clearpage
\input{supplementaryINPUT.tex}





\end{document}

%% file: supplementaryINPUT.tex
\appendix
\twocolumn[\section*{{\LARGE Supplementary Material}\vspace{2mm} \\ {\Large Attentive Group Equivariant Convolutional Networks}} \vspace{4mm}]

\section{Generalized Visual Self-Attention} 
Before we derive the constraints for general visual self-attention and prove Thm.~1 of the main article, we first motivate our definition of group equivariant visual self-attention. In the subsequent subsections we explain that our definition of attentive group convolution, as given in Eq. 14 of the main article, and reformulated in Eq.~\ref{eq:attentivegconv}, essentially describes a group equivariant linear mapping that is augmented with an additional attention function. 

\subsection{Self-attention: From Vectors to Feature Maps}
Let us first consider the general form of a linear map between respectively vector spaces (used in multi-layer perceptrons) and feature maps (used in (group) convolutional neural nets), defined as follows:
\begin{align}
\label{eq:linmapvectors}
\text{vectors:}
\hspace{0.6cm} 
\mathbf{x}^{out}_c &= \sum_{\tilde{c}}^{N_{\tilde{c}}} \mathbf{W}_{c,\tilde{c}} \, \mathbf{x}^{in}_{\tilde{c}},  \\
\label{eq:linmapfeaturemaps}
\text{feat maps:} \hspace{0.1cm} 
f^{out}_c(g) &= \sum_{\tilde{c}}^{N_{\tilde{c}}} \int\limits_G \Psi_{c,\tilde{c}}(g,\tilde{g}) f^{in}_{\tilde{c}}(\tilde{g}) {\rm d}\tilde{g}
\end{align}
Here, the first equation describes a linear map between vectors $\mathbf{x}^{in} \in \mathbb{R}^{N_{\tilde{c}}}$ and $\mathbf{x}^{out} \in \mathbb{R}^{N_{c}}$ via matrix-vector multiplication with matrix $\mathbf{W} \in \mathbb{R}^{N_{c} \times N_{\tilde{c}}}$. The second equation describes a linear map between feature maps $f^{in} \in (\mathbb{L}_2(G))^{N_{\tilde{c}}}$ and $f^{out} \in (\mathbb{L}_2(G))^{N_{c}}$, via a two argument kernel $\Psi \in \mathbb{L}_1(G \times G)^{N_{\tilde{c}} \times N_{c}}$. The two argument kernel $\Psi$ can be seen as the continuous counterpart of the matrix $\mathbf{W}$, and matrix-vector multiplication (sum over input indices) is augmented with an integral over the input coordinates $\tilde{g}$.

Keeping this form of linear mapping, we define the self-attentive map as the regular linear map augmented with attention weights computed from the input. Consequently, we formally define the self-attentive mappings as:
\begin{align}
\label{eq:generalattentionvectors}
\text{vectors:} \hspace{0.6cm} 
\mathbf{x}^{out}_c &= \sum_{\tilde{c}}^{N_{\tilde{c}}} \mathbf{A}_{c,\tilde{c}}\mathbf{W}_{c,\tilde{c}} \, \mathbf{x}^{in}_{\tilde{c}}, \\
\label{eq:generalattentionfeaturemaps}
\text{feat maps:} \hspace{0.1cm} 
f^{out}_c(g) &= \sum_{\tilde{c}}^{N_{\tilde{c}}} \int\limits_G \alpha_{c,\tilde{c}}(g,\tilde{g}) \Psi_{c,\tilde{c}}(g,\tilde{g}) \\[-5mm]
& \hspace{35mm} 
f^{in}_{\tilde{c}}(\tilde{g}) {\rm d}\tilde{g} \nonumber
\end{align}
in which the attention weights are computed from the input via some operator $\mathcal{A}$, i.e., $\mathbf{A}_{c,\tilde{c}} = \mathcal{A}[\mathbf{x}^{in}]_{c,\tilde{c}}$ in the vector case and $\alpha_{c,\tilde{c}} = 
\mathcal{A}[f^{in}]_{c,\tilde{c}}$ in the case of feature maps.

\subsection{Equivariant Linear Maps are Group Convolutions}
\label{sec:gconvderivation}
Now, since we want to preserve the spatial correspondences between the input and output feature maps, special attention should be paid to the continuous self-attentive mappings. In other words, these operators should be equivariant. By including an equivariance constraint on the linear mapping of Eq.~\ref{eq:linmapfeaturemaps} we obtain a group convolution (see e.g. \citet{kondor2018generalization,cohen2019general,bekkers2020bspline}).\break The derivation is as follows: 

Imposing the equivariance constraint $
\mathcal{L}_{{g}}[f^{in}] \underset{\text{Eq.} \ref{eq:linmapfeaturemaps}}{\mapsto} \mathcal{L}_{{g}}[f^{out}]$ means that for all $\overline{g},g \in G$ and all $f \in \mathbb{L}_2(G)^{N_c}$ we must guarantee that:
\begin{multline*}
   \hspace{2.5cm} \mathcal{L}_{{g}}[f^{in}] = \mathcal{L}_{{g}}[f^{out}]\\
\Leftrightarrow \\
\int_G \Psi_{c,\tilde{c}}(g,\tilde{g}) \mathcal{L}_{\overline{g}}\left[f\right](\tilde{g}) {\rm d}\tilde{g}
=
\int_G \Psi_{c,\tilde{c}}(\overline{g}^{-1} g, \tilde{g}) f(\tilde{g}) {\rm d}\tilde{g} \\
\Leftrightarrow \\
\int_G \Psi_{c,\tilde{c}}(g,\tilde{g}) f(\overline{g}^{-1}\tilde{g}) {\rm d}\tilde{g}
=
\int_G \Psi_{c,\tilde{c}}(\overline{g}^{-1} g, \tilde{g}) f(\tilde{g}) {\rm d}\tilde{g} \\
\Leftrightarrow \\
\int_G \Psi_{c,\tilde{c}}(g,\tilde{g}) f(\overline{g}^{-1}\tilde{g}) {\rm d}\tilde{g}
= \hspace{3.5cm}\\\hspace{2.5cm}
\int_G \Psi_{c,\tilde{c}}(\overline{g}^{-1} g, \overline{g}^{-1}\tilde{g}) f(\overline{g}^{-1}\tilde{g}) {\rm d}\tilde{g},
\end{multline*}
where the change of variables $\tilde{g}\rightarrow\overline{g}^{-1}\tilde{g}$ as well as the left-invariance of the Haar measure ( ${\rm d}(\overline{g}^{-1}\tilde{g}) = {\rm d}\tilde{g}$)) is used in the last step. Since this equality must hold for all $f\in\mathbb{L}_2(G)^{N_{\tilde{c}}}$ we obtain that $\Psi$ should be left-invariant in both input arguments. In other words, we have that $$
\forall_{\overline{g}.\in G}: \Psi(\overline{g}g,\overline{g}\tilde{g}) = \Psi(g,\tilde{g})
$$ 
Resultantly, we can always multiply both arguments with $g^{-1}$ and obtain $\Psi(e,g^{-1} \tilde{g})$, which is effectively a single argument function $\psi(g^{-1}\tilde{g}):=\Psi(e,g^{-1}\tilde{g})$ that takes as input a relative \enquote{displacement} $g^{-1}\tilde{g}$. Consequently, under the equivariance constraint, Eq.~\ref{eq:linmapfeaturemaps} becomes a group convolution:
$$
f^{out}_c(g) = \sum_{\tilde{c}}^{N_{\tilde{c}}} \int\limits_G \psi_{c,\tilde{c}}(g^{-1}\tilde{g}) f^{in}_{\tilde{c}}(\tilde{g}) {\rm d}\tilde{g}. 
$$

\subsection{Proof of Theorem 1}
We can apply the same type of derivation to reduce the general form of visual self-attention of Eq.~\ref{eq:generalattentionfeaturemaps} to our main definition of attentive group convolution:
\begin{equation}
\label{eq:attentivegconv}
f^{out}_c(g) = \sum_{\tilde{c}}^{N_{\tilde{c}}} \int\limits_G \alpha_{c,\tilde{c}}(g,\tilde{g})\psi_{c,\tilde{c}}(g^{-1}\tilde{g}) f^{in}_{\tilde{c}}(\tilde{g}) {\rm d}\tilde{g}.
\end{equation}
However, we cannot reduce attention map $\alpha$ to a single argument function like we did for the kernel $\Psi$ since $\alpha$ depends on the input $f^{in}$. To see this consider the following:

Without loss of generality, let $\mathfrak{A}: \mathbb{L}_2(G) \rightarrow \mathbb{L}_2(G)$ denote the attentive group convolution defined by Eq.~\ref{eq:attentivegconv}, with $N_c = N_{\tilde{c}} = 1$, and some $\psi$ which in the following we omit in order to simplify our derivation. Equivariance of $\mathfrak{A}$ implies that $\forall_{f \in \mathbb{L}_2(G)}$, $ \forall_{\overline{g},g \in G}:$
\begin{gather*}
\mathfrak{A}\left[\mathcal{L}_{\overline{g}}\left[f\right]\right](g) = \mathcal{L}_{\overline{g}}\left[\mathfrak{A}\left[f\right]\right](g) \\
\Leftrightarrow \\
\mathfrak{A}\left[\mathcal{L}_{\overline{g}}\left[f\right]\right](g) = \mathfrak{A}\left[f\right](\overline{g}^{-1} g) \\
\Leftrightarrow \\
\int_G \mathcal{A}\left[\mathcal{L}_{\overline{g}}\left[f\right]\right](g,\tilde{g})\mathcal{L}_{\overline{g}}\left[f\right](\tilde{g}) {\rm d}\tilde{g} 
= \hspace{2.2cm}\\\hspace{3.5cm}
\int_G \mathcal{A}\left[f\right](\overline{g}^{-1}g,\tilde{g}) f(\tilde{g}) {\rm d}\tilde{g} \\
\Leftrightarrow \\
\int_G \mathcal{A}\left[\mathcal{L}_{\overline{g}}\left[f\right]\right](g,\tilde{g})f(
\overline{g}^{-1}\tilde{g}) {\rm d}\tilde{g} 
= \hspace{1.9cm}\\\hspace{2.5cm}
\int_G \mathcal{A}\left[f\right](\overline{g}^{-1}g,\overline{g}^{-1}\tilde{g}) f(\overline{g}^{-1}\tilde{g}) {\rm d}\tilde{g},
\end{gather*}
where we once again perform the variable substitution $\tilde{g}\rightarrow \overline{g}^{-1}\tilde{g}$ at the right hand side of the last step. This must hold for all $f \in \mathbb{L}_2(G)$ and hence:
\begin{equation}
\label{eq:constraint}
\forall_{\overline{g}\in G}: \mathcal{A}\left[\mathcal{L}_{\overline{g}}f\right](g,\tilde{g}) = \mathcal{A}\left[f\right](\overline{g}^{-1}g,\overline{g}^{-1}\tilde{g}),
\end{equation}
which proves the constraint on $\mathcal{A}$ as given in Thm.~1 of the main article. Just as for convolutions in Sec.~\ref{sec:gconvderivation}, we can turn this into a single argument function as:
\begin{equation}
\label{eq:singleargatt}
\mathcal{A}[f](g,\tilde{g}) = \mathcal{A}\left[\mathcal{L}_{g^{-1}} f\right](e,g^{-1} \tilde{g}) =: \mathcal{A}'[\mathcal{L}_{g^{-1}} f](g^{-1}\tilde{g}),
\end{equation}
in which $\mathcal{A}'$ is an attention operator that generates a single argument attention map from an input $f$. However, this would mean that for each $g$ the input should be transformed via $\mathcal{L}_{g^{-1}}$, which does not make things easier for us. Things do get easier when we choose to attend to either the input \emph{or} the output, which we discuss next.
\begin{corollary}
\label{cor1}
Each attention operator $\mathcal{A}$ that generates an attention map $\alpha:G\times G \rightarrow [0,1]$ which is left-invariant to either one of the arguments, and thus exclusively attends either the input or output domain, satisfies the equivariance constraint of Eq.~\ref{eq:constraint}, iff the operator is G-equivariant, i.e., a group convolution.
\end{corollary}
\begin{proof}
Left-invariant to either one of the arguments (let us now consider invariance in the first argument) means that:
\begin{align*}
\forall_{g,\tilde{g}}: \;\;  \mathcal{A}[f](g,\tilde{g}) = \mathcal{A}[f]](e,\tilde{g}),
\end{align*}
and hence, we are effectively dealing with a single argument attention map, which we define as $\mathcal{A}'[f](\tilde{g}):=\mathcal{A}(e,\tilde{g})$. Consequently, the equivariance constraint of Eq.~\ref{eq:constraint} becomes:
\begin{align*}
\forall_{\overline{g}\in G}: \mathcal{A}\left[\mathcal{L}_{\overline{g}}f\right](g,\tilde{g}) &= \mathcal{A}\left[f\right](\overline{g}^{-1}g,\overline{g}^{-1}\tilde{g}) \Leftrightarrow \\
\forall_{\overline{g}\in G}: \mathcal{A}'\left[\mathcal{L}_{\overline{g}}f\right](\tilde{g}) &= \mathcal{A}'\left[f\right](\overline{g}^{-1}\tilde{g}) \Leftrightarrow \\
\forall_{\overline{g}\in G}: \mathcal{A}'\left[\mathcal{L}_{\overline{g}}f\right] &= \mathcal{L}_{\overline{g}}\left[\mathcal{A}'\right]\left[f\right].
\end{align*}
Conclusively, $\mathcal{A}'$ must be an equivariant operator.
\end{proof}
The derivation of the Eq.~\ref{eq:constraint} together with the proof of Corollary~\ref{cor1} completes the proof of Theorem 1 of the main article.
\subsection{Equivariance Proof of the Proposed Visual Attention}
In this section we revisit the proposed attention mechanisms and prove that they indeed satisfy Thm.~1 of the main article. Recall the general formulation of attentive group convolution given in Eq.~\ref{eq:attentivegconv}. Inspired by the work of \citet{woo2018cbam}, we reduce the computation load by factorizing the attention map $\alpha$ into channel and spatial components via:
$$
\alpha_{c,\tilde{c}}(g,\tilde{g}) = \alpha^{\mathcal{X}}(x, h, \tilde{h})\alpha^{C}_{c,\tilde{c}}(h,\tilde{h})
$$
where $\alpha^{\mathcal{C}}$ attends to both input and output channels as well as input and output poses $h,\tilde{h}\in H$, and spatial attention attends to the output domain $g = (x,h) \in G$ for all input poses $\tilde{h} \in H$ but does not change for input spatial positions $\tilde{x} \in \mathbb{R}^{d}$. We denote the operators $\mathcal{A}^{C}$, $\mathcal{A}^{X}$ utilized to compute the attention maps as $\alpha^{\mathcal{C}} = \mathcal{A}^{\mathcal{C}}[f]$ and $\alpha^{\mathcal{X}} = \mathcal{A}^{\mathcal{X}}[f]$, respectively.

\subsubsection{Channel attention}
We compute channel attention via{\ifarxivcopy\null\fi:
\begin{align}
    \mathcal{A}^{\mathcal{C}}[f](h,\tilde{h})&=\varphi^{\mathcal{C}}\left[s^{\mathcal{C}}\left[\tilde{f}[f]\right]\right](h,\tilde{h}) \label{eq:compute_ak_2_suppl}\\
    &\hspace{-1.0cm} = \sigma\Big(\big[
    \mathbf{W}_{2}(h^{-1}\tilde{h})\cdot[\mathbf{W}_{1}(h^{-1}\tilde{h})\cdot s^{\mathcal{C}}_{ \text{avg}}(h,\tilde{h})]^{+}
    \big] \nonumber  \\[-2\jot]
   &\hspace{0.2cm}
   + \big[
    \mathbf{W}_{2}(h^{-1}\tilde{h})\cdot[\mathbf{W}_{1}(h^{-1}\tilde{h})\cdot s^{\mathcal{C}}_{ \text{max}}(h,\tilde{h})]^{+}
    \big] \Big) \nonumber
\end{align}
with 
\begin{equation}\label{eq:intermediate}
    \tilde{f}_{c,\tilde{c}}(x,h,\tilde{h}) := \big[f_{\tilde{c}} \star_{\mathbb{R}^{d}} \mathcal{L}_{h}[\psi_{c,\tilde{c}}]\big](x, \tilde{h})
\end{equation} the intermediary result from the convolution between the input $f$ and the $h$-transformation of the filter $\psi$, $\mathcal{L}_{h}[\psi]$ before pooling over $\tilde{c}$ and $\tilde{h}$. $s^{\mathcal{C}}_{\text{avg}}$ and $s^{\mathcal{C}}_{\text{max}}$ denote respectively average and max pooling over the ${x}$ coordinate. 

Here, we apply a slight abuse of notation with $\tilde{f}[f]$ and $s^{\mathcal{C}}[\tilde{f}]$ in order to keep track of the dependencies. In order to proof equivariance of the attention operator $\mathcal{A}^{\mathcal{C}}$ we need to proof that $\forall_{\overline{g} \in G}: \mathcal{A}^{\mathcal{C}}[\mathcal{L}_{\overline{g}}[ f]](h,\tilde{h}) = \mathcal{A}^{\mathcal{C}}[f](\overline{h}^{-1} h,\overline{h}^{-1} \tilde{h})$, with $\overline{g} = (\overline{x},\overline{h})$. To this end, we first identify the equivariance and invariance properties of the functions used in Eq.~\ref{eq:compute_ak_2_suppl}. 

From Eq.~\ref{eq:intermediate} we see that the intermediate convolution result $\tilde{f}$ is equivariant via $\tilde{f}[\mathcal{L}_{\overline{g}}[f]](\tilde{x},\tilde{h}) =  \tilde{f}[f](\overline{g}^{-1} x, \overline{h}^{-1} h, \overline{h}^{-1} \tilde{h})$. For the statistics operators $s^{\mathcal{C}}$ we have invariance w.r.t. translations due to the pooling over $x$, and equivariance w.r.t. parameter $\overline{h}$ via $s^{\mathcal{C}}[\tilde{f}[\mathcal{L}_{\overline{g}}{f}]](h,\tilde{h}) = s^{\mathcal{C}}[\tilde{f}[f]](\overline{h}^{-1} h,\overline{h}^{-1}\tilde{h})$. Now, we propagate the transformation on the input and compute the result of $\mathcal{A}^{\mathcal{C}}[\mathcal{L}_{\overline{g}}[f]](g,\tilde{g})$. That is, we compute the left-hand side of the constraint given in Eq.~\ref{eq:constraint}, where, for brevity, we omit the $s^{\mathcal{C}}_{\text{max}}$ term:
\begin{multline*}
\mathcal{A}^{\mathcal{C}}[\mathcal{L}_{\overline{g}}[f]](g,\tilde{g}) = \\
    \mathbf{W}_{2}(h^{-1}\tilde{h})\cdot[\mathbf{W}_{1}(h^{-1}\tilde{h})\cdot s^{\mathcal{C}}_{ \text{\text{avg}}}(\overline{h}^{-1}h,\overline{h}^{-1}\tilde{h})]^{+}.
\end{multline*}
The right-hand side of Eq.~\ref{eq:constraint} is given by:
\begin{multline*}
\mathcal{A}^{\mathcal{C}}[f](\overline{g}^{-1}g,\overline{g}^{-1}\tilde{g}) = \\
    \mathbf{W}_{2}(h^{-1}\tilde{h})\cdot[\mathbf{W}_{1}(h^{-1}\tilde{h})\cdot s^{\mathcal{C}}_{ \text{\text{avg}}}(\overline{h}^{-1}h,\overline{h}^{-1}\tilde{h})]^{+}.
\end{multline*}
and hence, Eq.~\ref{eq:constraint} is satisfied for all $\overline{g} \in G$. Resultantly, $\mathcal{A}^{\mathcal{C}}$ is a valid attention operator.

\subsubsection{Spatial attention}

The spatial attention map $\alpha^{\mathcal{X}}$ is computed via: 
\begin{align}
\alpha^{\mathcal{X}}(g, \tilde{h}) &= \mathcal{A}^{\mathcal{X}}[f](g, \tilde{h}) \nonumber \\
&= \varphi^{\mathcal{X}}\left[s^{\mathcal{X}}\left[\tilde{f}[f]\right]\right](g,\tilde{h})\nonumber\\
&=\sigma\left([s^{\mathcal{X}} \star_{\mathbb{R}^{d}} \mathcal{L}_{h}[\psi^{\mathcal{X}}]\right)(x, \tilde{h}) \label{eq:spat_att}
\end{align}
where $\sigma$ is a point-wise logistic sigmoid, $\psi^{\mathcal{X}}: G \rightarrow \mathbb{R}^2$ is a group convolution filter and $s^{\mathcal{X}}[\tilde{f}]:G \times H \rightarrow \mathbb{R}^2$ is a map of averages and maximum values taken over the channel axis at each $g \in G$ in $\tilde{f}$ for each $\tilde{h} \in H$. Note that Eq.~\ref{eq:spat_att} corresponds to a group convolution up to the final pooling operation over $\tilde{h}$. Since the statistics operator $s^{\mathcal{X}}$ is invariant w.r.t. translations in the input and Eq.~\ref{eq:spat_att} corresponds to a group convolution (up to pooling over $\tilde{h}$), we have that $\mathcal{A}^{X}$ is a valid attention operator as well.

\section{Extended Implementation Details}\label{appx:extended_details}
In this section we provide extended details over our implementation. For the sake of completeness and reproducibility, we summarize the exact training procedures utilized during our experiments. Moreover, we delve into some important changes performed to some network architectures during our experiments to ensure \textit{exact} equivariance, and shed light into their importance for our equivariant attention maps.

\subsection{General Observations}

We utilize \texttt{PyTorch} for our implementation. Any missing parameter specification in the following sections can be safely considered to be the default value of the corresponding parameter. For batch normalization layers, we utilize eps=0.00002 similarly to \citet{GCNN}.

\subsection{rot-MNIST}
For rotational MNIST, we utilized the same backbone network as in \citet{GCNN}. During training we utilize Adam \cite{kingma2014adam}, batches of size 128, weight decay of 0.0001, learning rate of 0.001, drop-out rate of 0.3 and perform training for 100 epochs. Importantly and contrarily to \citet{GCNN}, we consistently experience improvements when utilizing drop-out and therefore we do not exclude it for any model. 

\subsection{CIFAR-10}
It is not clear from \citet{springenberg2014striving, GCNN} which batch size is used in their experiments. For our experiments, we always utilize batches of size 128.

\subsubsection{All-CNN}
We utilize the All-CNN-C structure of \citet{springenberg2014striving}. Analogously to \citet{springenberg2014striving, GCNN}, we utilize stochastic gradient descent, weight decay of 0.001 and perform training for 350 epochs. We utilize a grid search on the set $\{$0.01, 0.05, 0.1, 0.25$\}$ for the learning rate and report the best obtained performance. Furthermore, we reduce the learning rate by a factor of 10 at epochs 200, 250 and 300.

\subsubsection{ResNet44}
Similar to \citet{GCNN}, we utilize stochastic gradient descent, learning rate of 0.05 and perform training for 300 epochs. Furthermore, we reduce the learning rate by a factor of 10 at epochs 50, 100 and 150.

\subsection{PCam}
During training on the PatchCamelyon dataset, we utilize Adam \cite{kingma2014adam}, batches of size 64, weight decay of 0.0001, learning rate of 0.001 and perform training for 100 epochs. Furthermore, we reduce the learning rate by a factor of 2 after 20 epochs of no improvement in the validation loss. 

\section{Effects of Stride and Input Size on Equivariance}\label{sec:approx_equiv}
Theoretically seen, the usage of stride during pooling and during convolution is of no relevance for the equivariance properties of the corresponding mapping \cite{GCNN}. However, we see that in practice stride can affect equivariance for specific cases as is the case for our experiments on CIFAR-10.

Consider the convolution between an input of even size and a small $3$x$3$ filter as shown in Fig.~\ref{fig:norot}. Via group convolutions, we can ensure that the output of the original input and a rotated one (Fig.~\ref{fig:yesrot}) will be exactly equal (up to the same rotation). Importantly however, note that for Fig. \ref{fig:strideequiv}, the local support of the filter, i.e., the input section with which the filter is convolved at a particular position, \textit{is not equivalent} for rotated versions of the input (denoted by blue circles for the non-rotated case and by green circles by for t
he rotated case). As a result, despite the group convolution itself being equivariant, the responses of both convolutions do not entirely resemble one another and, consequently, the depicted strided group convolution is \textit{not} exactly equivariant.

It is important to highlight that this behaviour is just exhibited for the special case when the residual between the used stride and the input size is even. Unfortunately, this is the case both for the ResNet44 as well as the All-CNN networks utilized in our CIFAR-10 experiments. However, as neighbouring pixels are extremely correlated with one another, the effects of this phenomenon are not of much relevance for the classification task itself. As a matter of fact, it can be interpreted as a form of data augmentation by skipping intermediary pixel values. Consequently, we can say that these networks are \textit{approximately equivariant}. 

Importantly, this phenomenon does affect the resulting equivariant attention maps generated via attentive group convolutions as shown in Fig.~\ref{fig:bad_examples}. As these networks are only equivariant in an approximate manner, the generated attention maps are slightly deformed versions of one another for multiple orientations. In order to alleviate this problem, we replace all strided convolutions in the All-CNN and ResNet44 architectures by conventional convolutions (stride=1), followed by spatial max pooling. Resultantly, we are able to produce exactly equivariant attention maps as shown in Fig. 7 in the main text and Fig.~\ref{fig:good_examples} here.

\begin{figure}[t!]
    \centering
    \begin{subfigure}{\textwidth}
    \centering
    \includegraphics[width=0.8\textwidth]{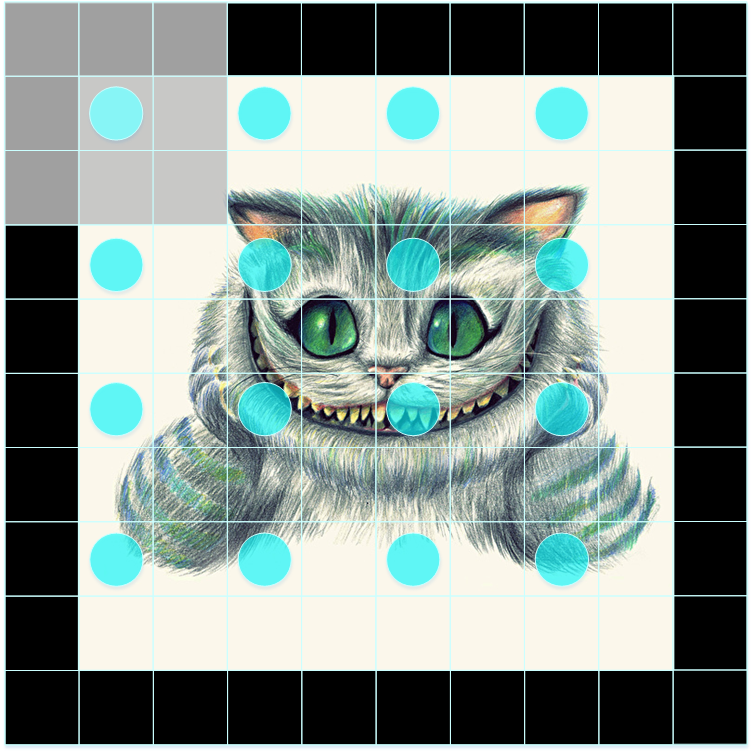}
    \caption{}\label{fig:norot}
    \end{subfigure}
    
    \begin{subfigure}{\textwidth}
    \centering
        \includegraphics[width=0.8\textwidth]{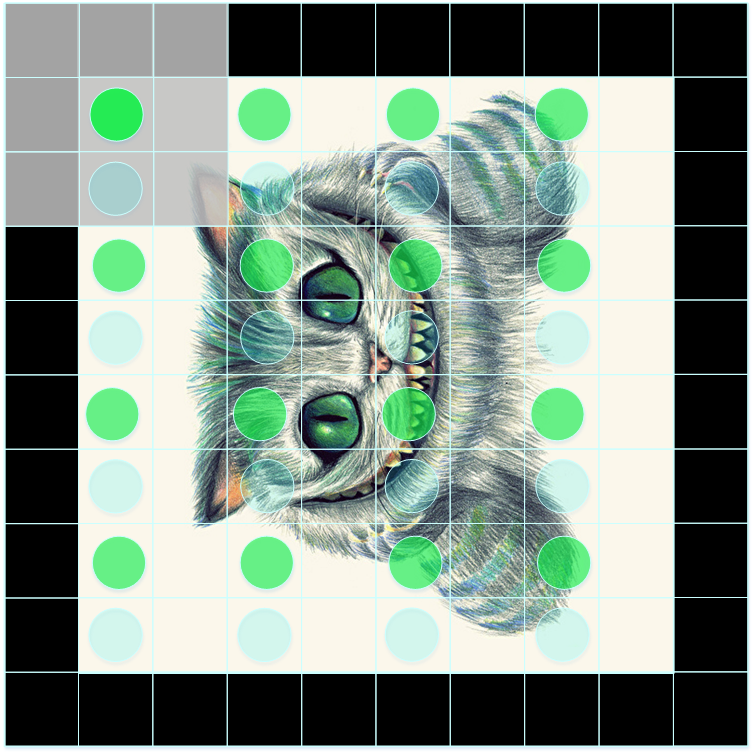}
        \caption{}\label{fig:yesrot}
    \end{subfigure}
    \caption{Effect of stride and input size on exact equivariance. Although group convolutions are ensured to be group equivariant, in practice, if the residual between the stride and the input size is even, as it's the case for the networks utilized in the CIFAR-10 experiments, equivariance is only approximate. This has important effects on equivariant attention maps (Fig.~\ref{fig:bad_examples}).}\label{fig:strideequiv}
\end{figure}

\begin{figure*}
    \centering
    \begin{subfigure}{0.17\textwidth}
        \includegraphics[width=\textwidth]{images/x_0_rot.png}
    \end{subfigure}
    \hspace{1mm}
    \begin{subfigure}{0.17\textwidth}
        \includegraphics[width=\textwidth]{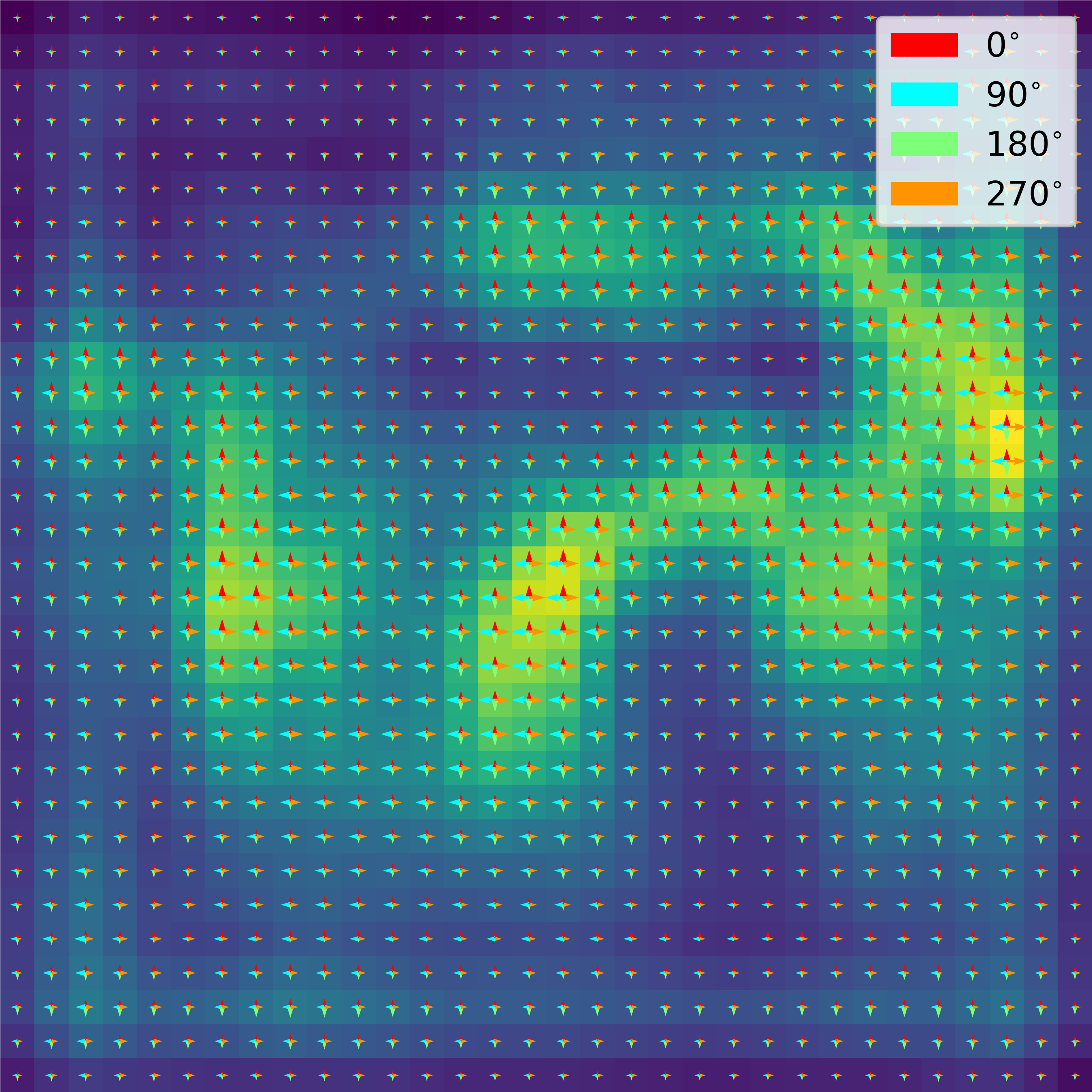}
    \end{subfigure}
    
    \vspace{6mm}
    \begin{subfigure}{0.17\textwidth}
        \includegraphics[width=\textwidth]{images/x_90_rot.png}
    \end{subfigure}
    \hspace{1mm}
    \begin{subfigure}{0.17\textwidth}
        \includegraphics[width=\textwidth]{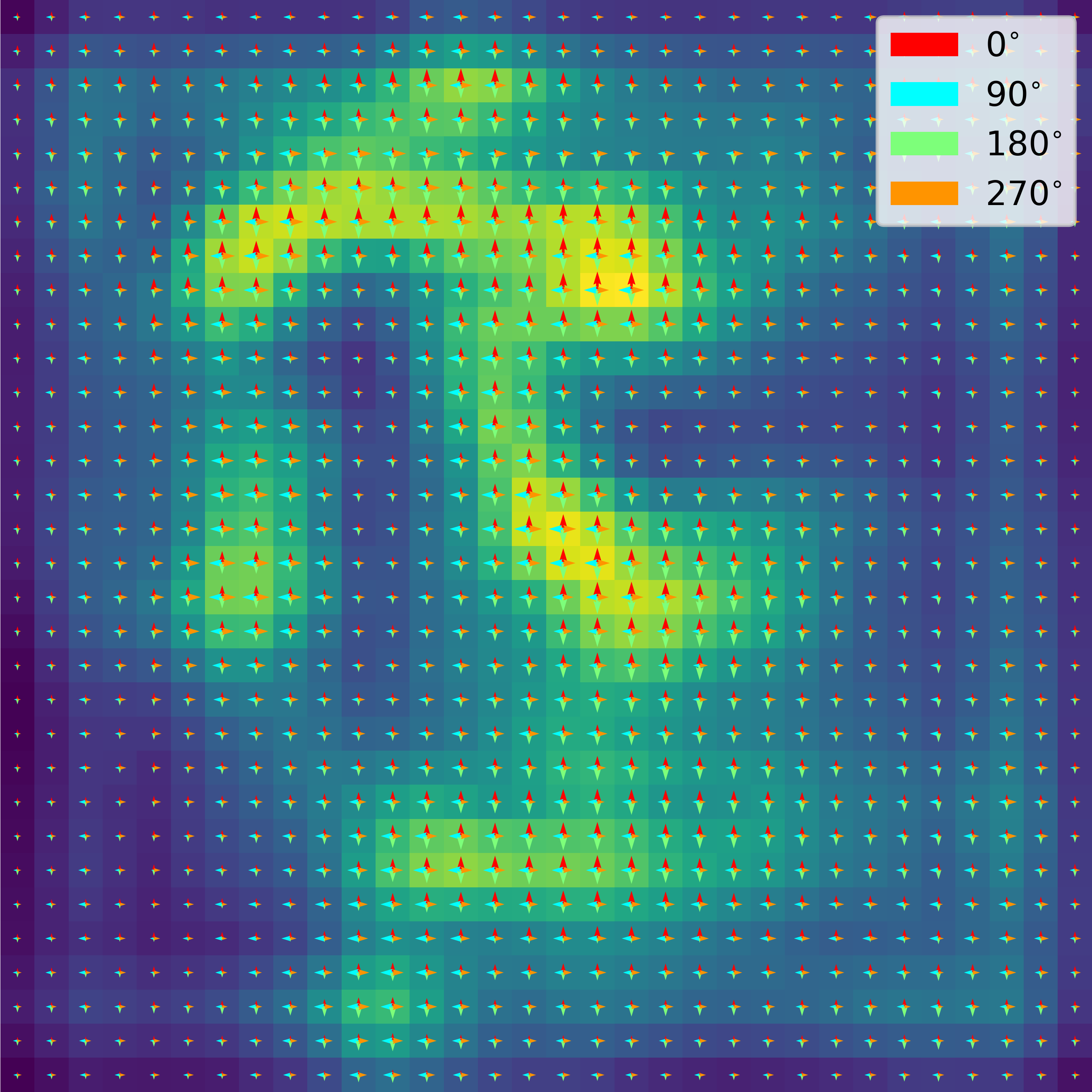}
    \end{subfigure}
    \hfill
    \begin{subfigure}{0.17\textwidth}
        \includegraphics[width=\textwidth]{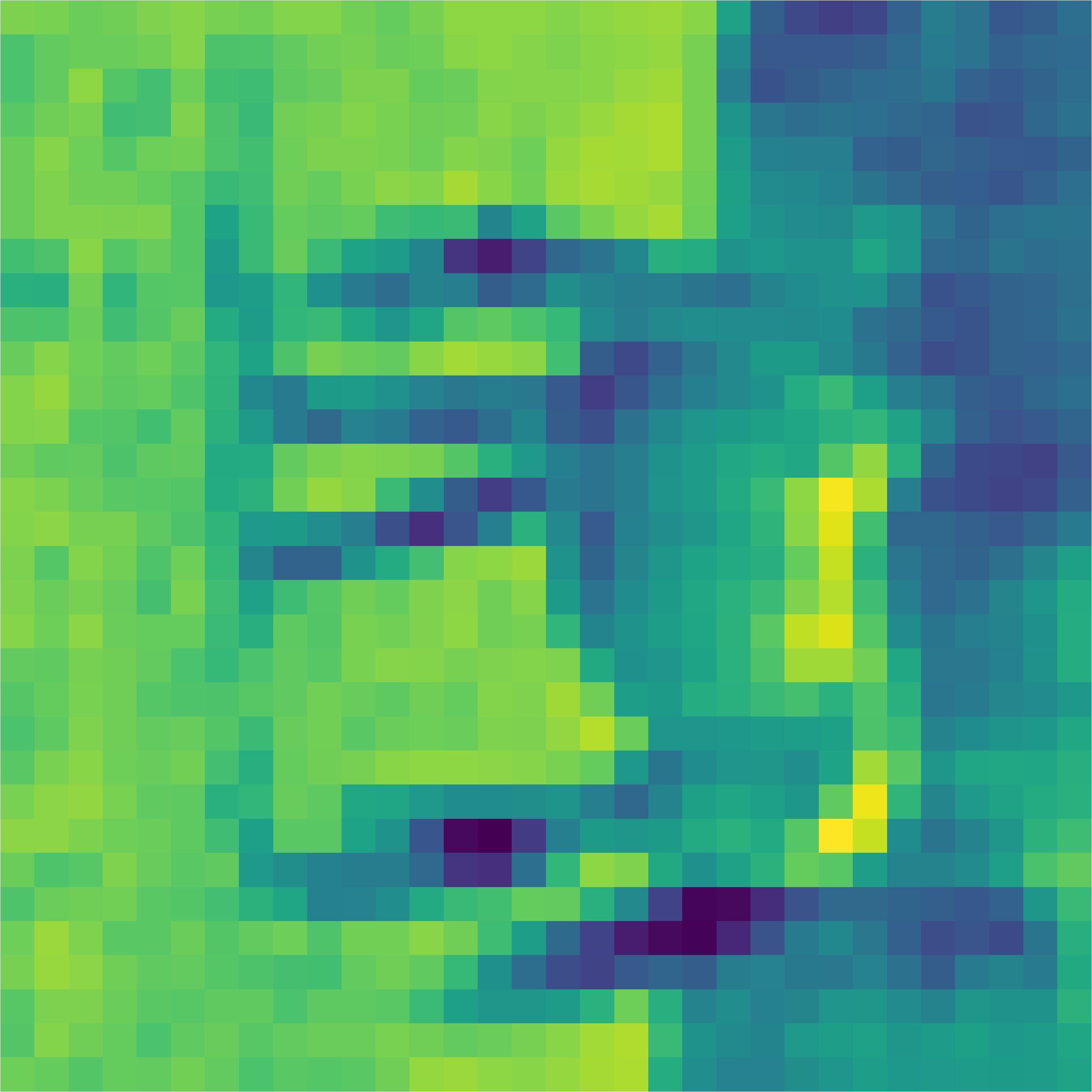}
    \end{subfigure}
    \hspace{1mm}
    \begin{subfigure}{0.17\textwidth}
        \includegraphics[width=\textwidth]{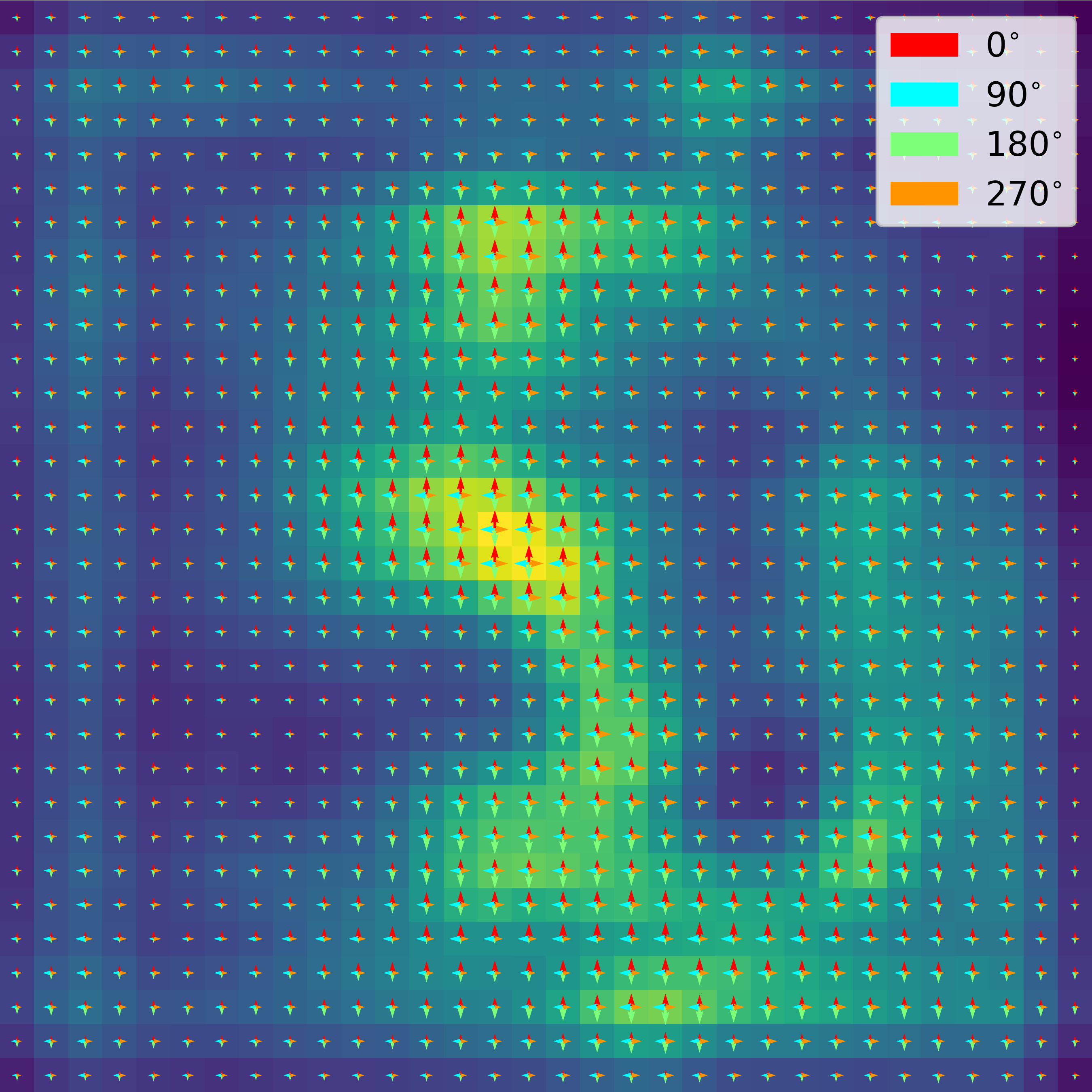}
    \end{subfigure}
    
    \vspace{6mm}
    \centering
    \begin{subfigure}{0.17\textwidth}
        \includegraphics[width=\textwidth]{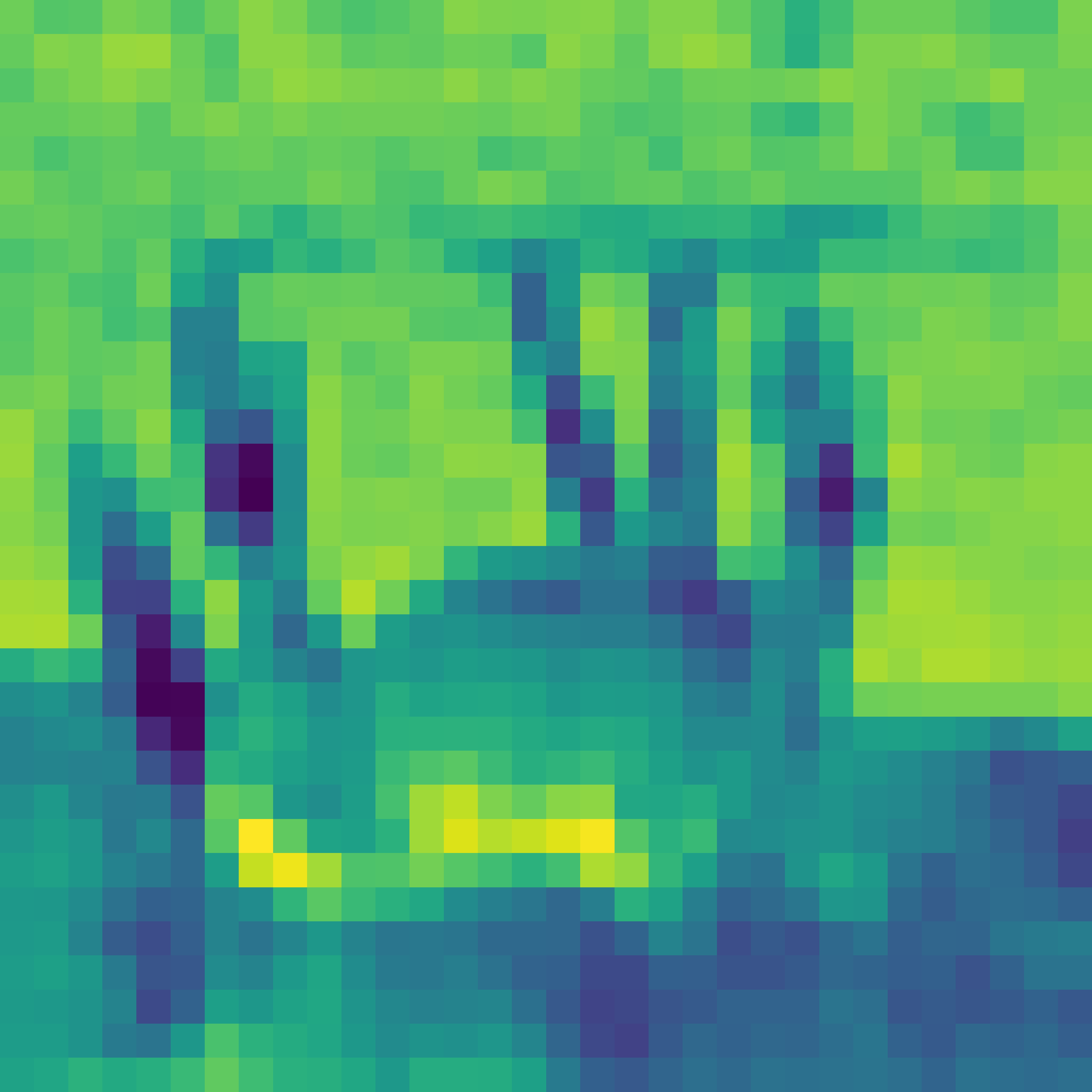}
    \end{subfigure}
    \hspace{1mm}
    \begin{subfigure}{0.17\textwidth}
        \includegraphics[width=\textwidth]{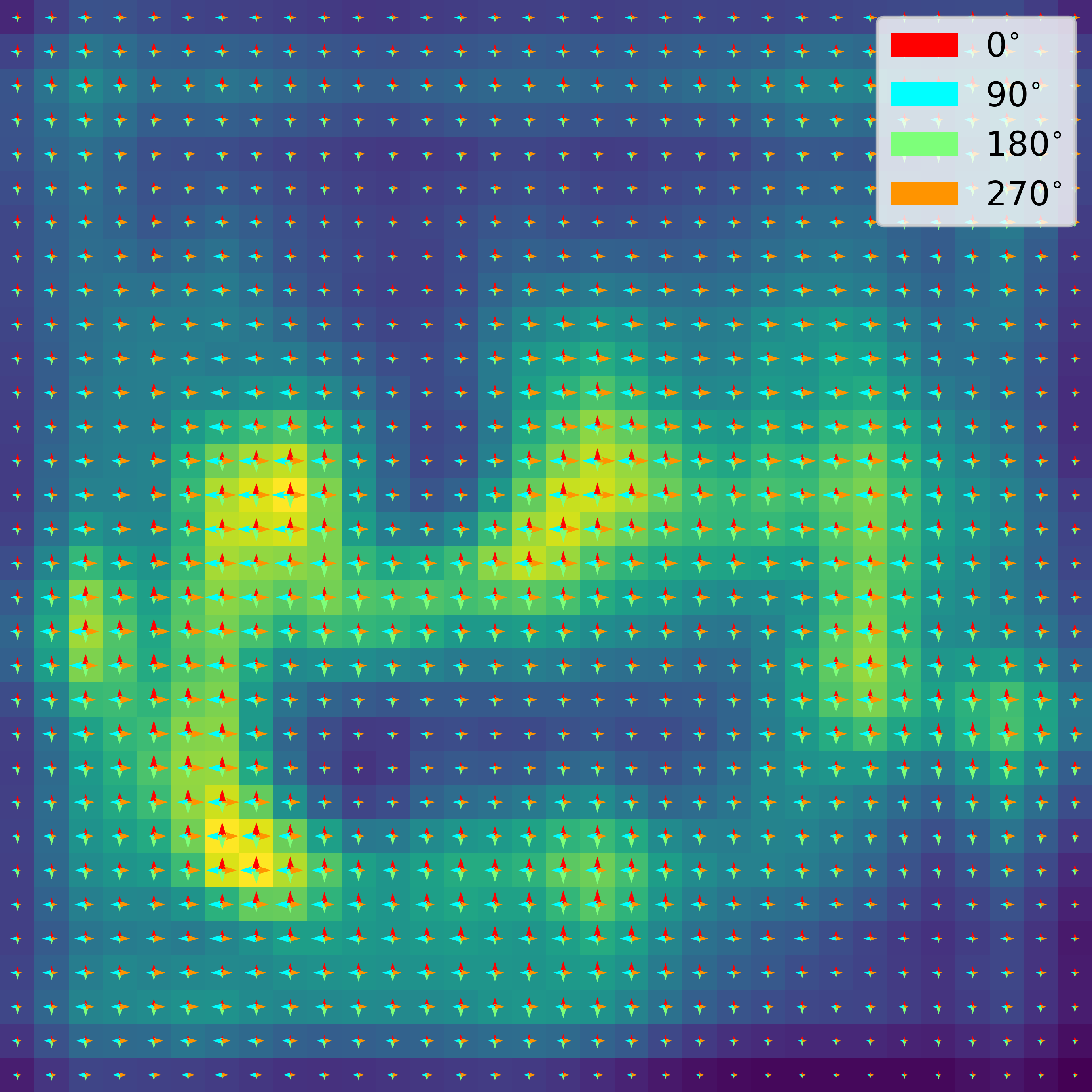}
    \end{subfigure}
    \caption{Examples of equivariant attention maps under the approximate equivariance regime. Note, for example, that attention around the horse's back changes for different orientations.}\label{fig:bad_examples}
\end{figure*}


\begin{figure*}
    \centering
    \begin{subfigure}{0.17\textwidth}
        \includegraphics[width=\textwidth]{images/x_0_rot.png}
    \end{subfigure}
    \hspace{1mm}
    \begin{subfigure}{0.17\textwidth}
        \includegraphics[width=\textwidth]{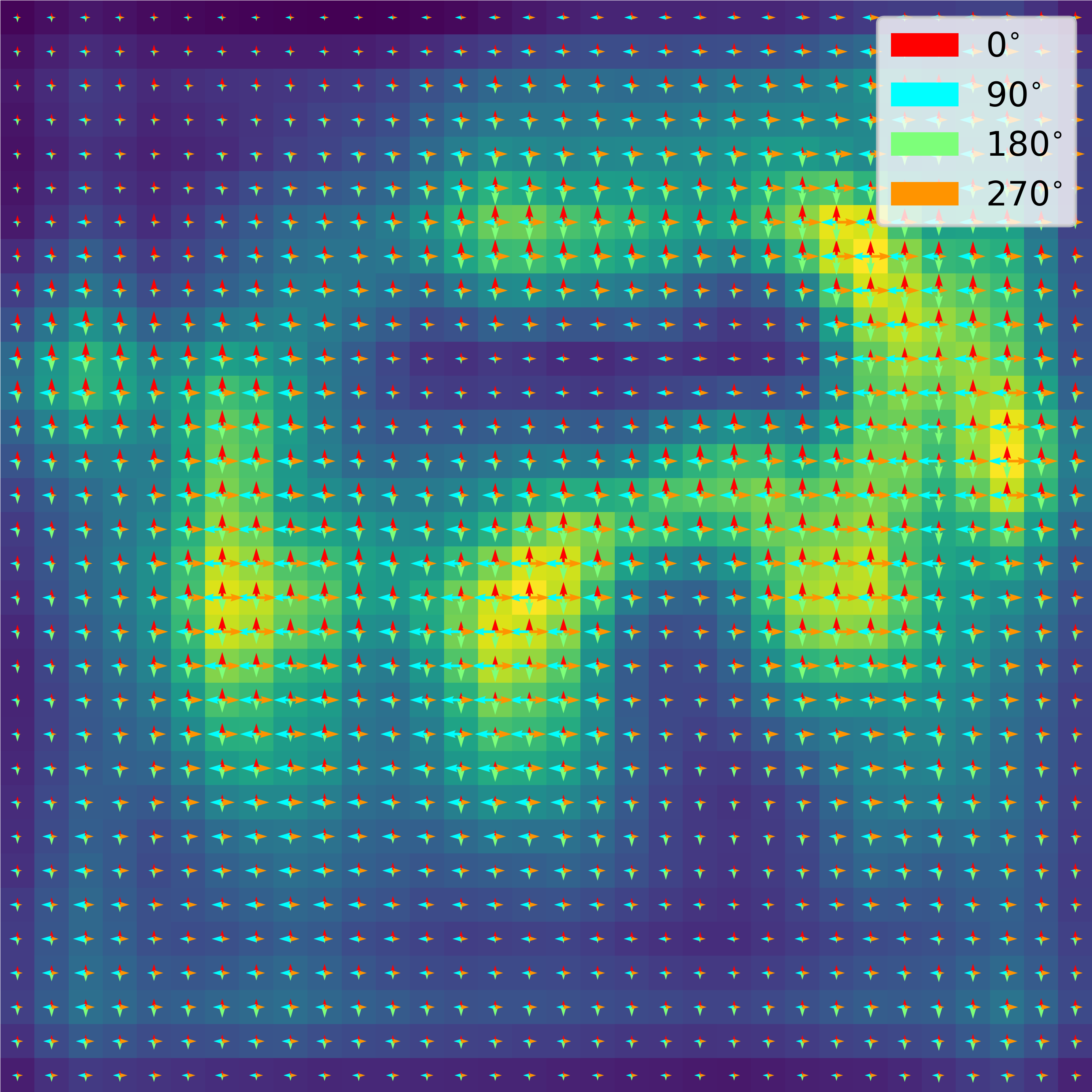}
    \end{subfigure}
    
    \vspace{6mm}
    \begin{subfigure}{0.17\textwidth}
        \includegraphics[width=\textwidth]{images/x_90_rot.png}
    \end{subfigure}
    \hspace{1mm}
    \begin{subfigure}{0.17\textwidth}
        \includegraphics[width=\textwidth]{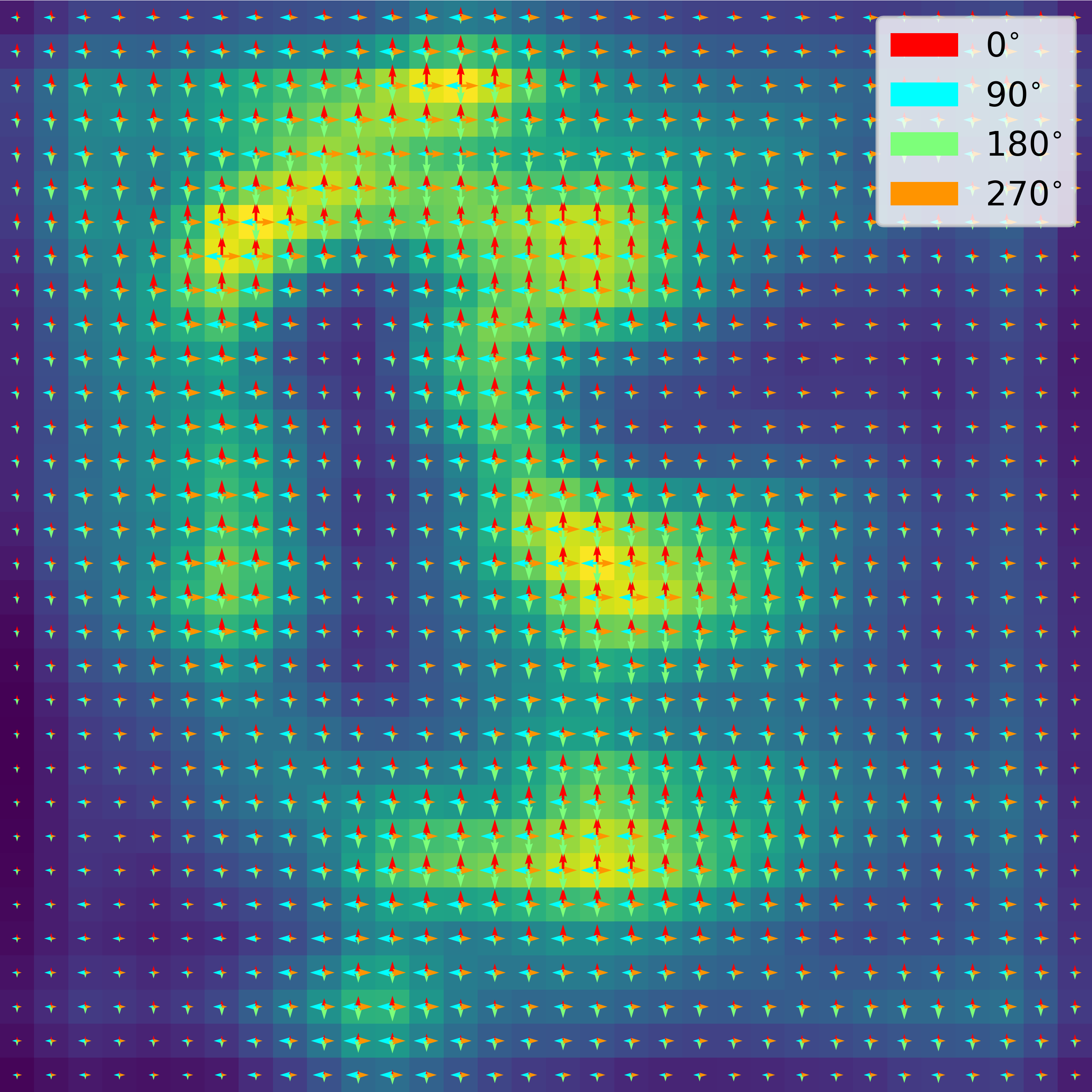}
    \end{subfigure}
    \hfill
    \begin{subfigure}{0.17\textwidth}
        \includegraphics[width=\textwidth]{images/x_270_rot.png}
    \end{subfigure}
    \hspace{1mm}
    \begin{subfigure}{0.17\textwidth}
        \includegraphics[width=\textwidth]{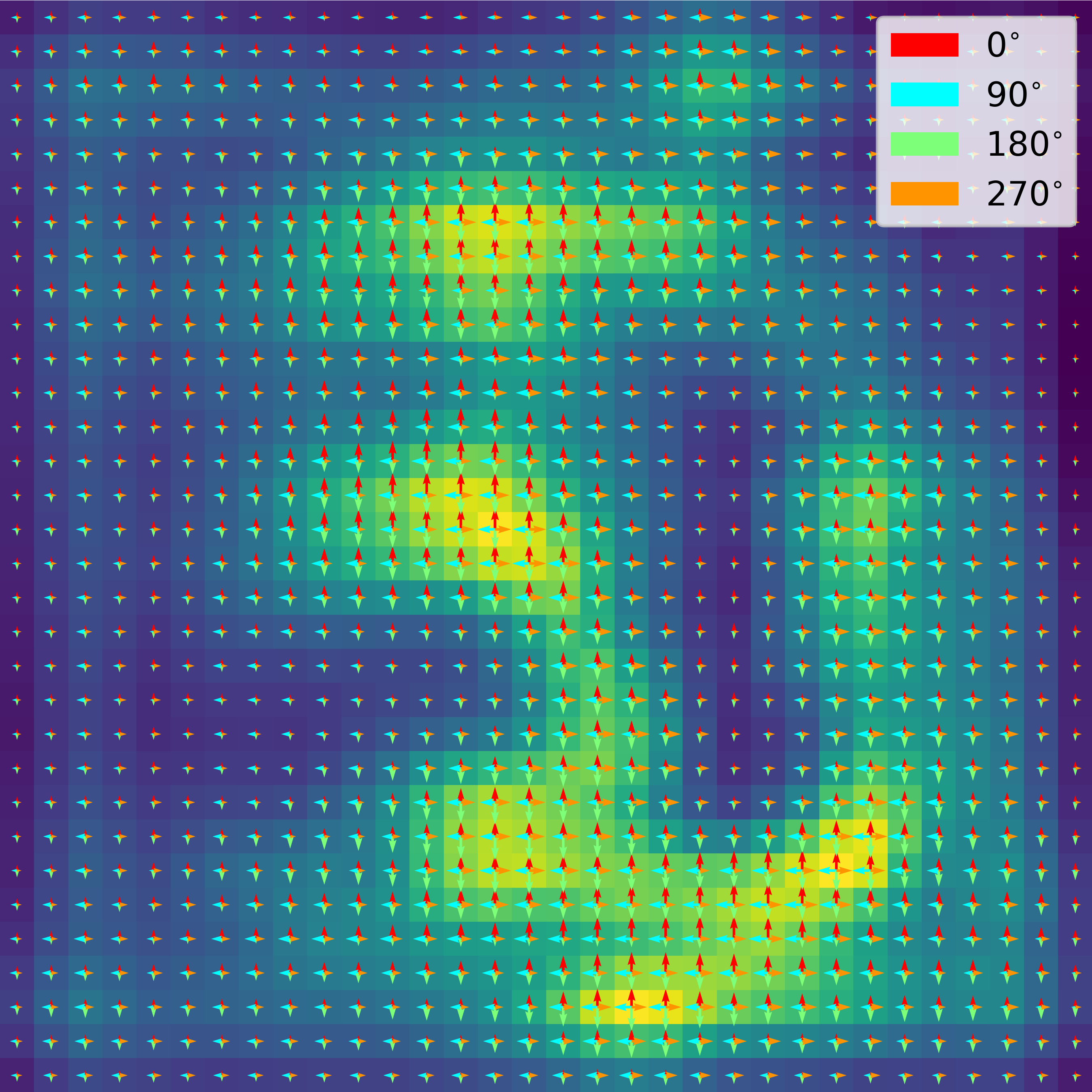}
    \end{subfigure}
    
    \vspace{6mm}
    \centering
    \begin{subfigure}{0.17\textwidth}
        \includegraphics[width=\textwidth]{images/x_180_rot.png}
    \end{subfigure}
    \hspace{1mm}
    \begin{subfigure}{0.17\textwidth}
        \includegraphics[width=\textwidth]{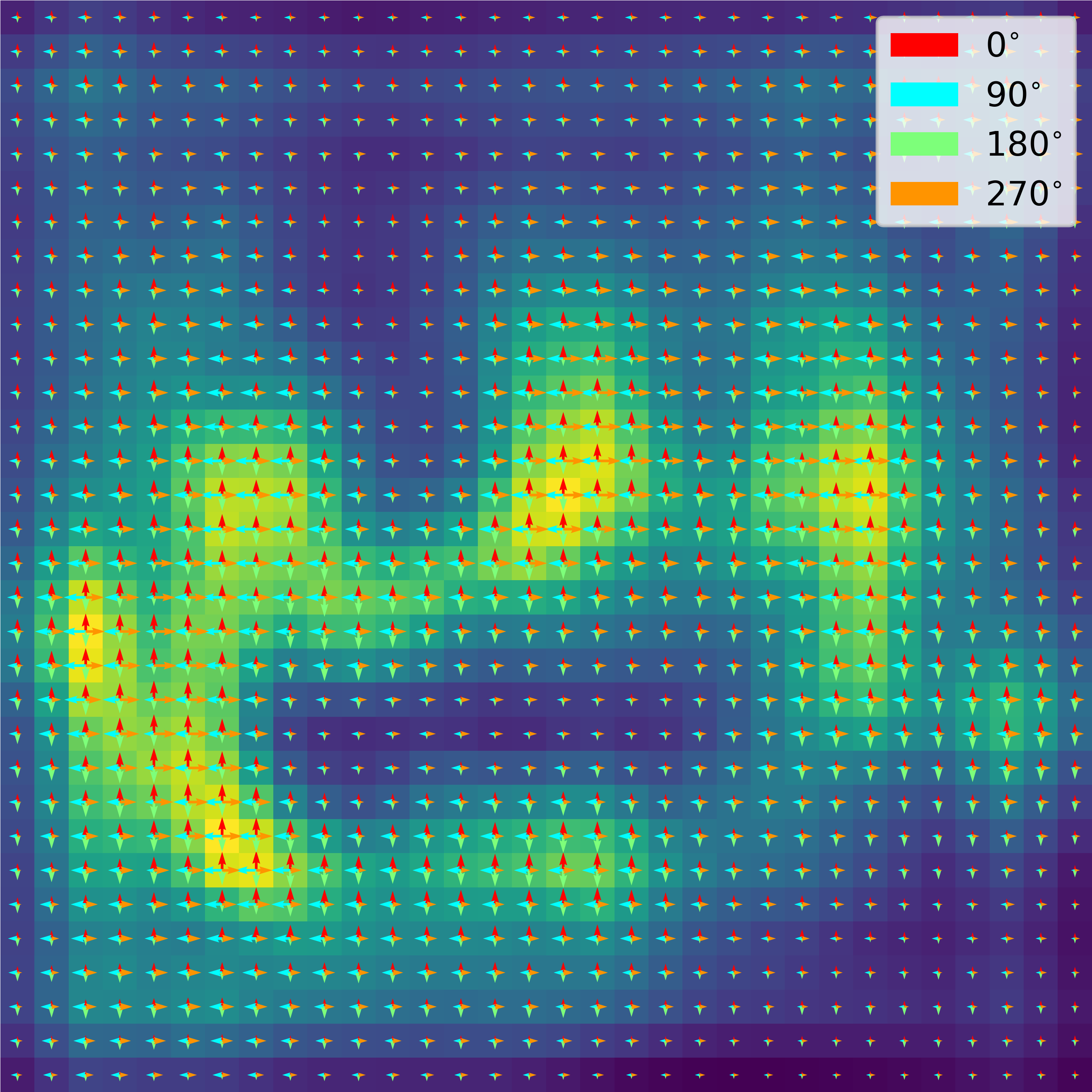}
    \end{subfigure}
    \caption{Examples of equivariant attention maps under the exact equivariance regime.}\label{fig:good_examples}
\end{figure*}